%% file: arxiv.tex
\input{_constants}
\arxiv 

\pdfoutput=1
\documentclass[10pt,twocolumn,letterpaper]{article}
\input{cvpr_header}
\usepackage{algorithm}
\usepackage{subcaption}
\usepackage{algpseudocode}
\usepackage{enumitem}
\usepackage{amsmath}
\usepackage{amsthm}
\usepackage{multirow}
\usepackage{adjustbox}
\usepackage[normalem]{ulem}
\newtheorem{theorem}{Theorem}[section]
\newcommand*{\rom}[1]{\expandafter{\romannumeral #1}}

\title{ACT-Diffusion: Efficient Adversarial Consistency Training for One-step Diffusion Models}

\author{
   {Fei Kong\textsuperscript{1} \quad Jinhao Duan\textsuperscript{2} \quad Lichao Sun\textsuperscript{3} \quad Hao Cheng\textsuperscript{4} \quad Renjing Xu\textsuperscript{4} }\\
  {Hengtao Shen\textsuperscript{1} \quad Xiaofeng Zhu\textsuperscript{1} \quad Xiaoshuang Shi\textsuperscript{1}\thanks{Equal corresponding author} \quad Kaidi Xu\textsuperscript{2\thinspace$*$}} \vspace{3pt}\\
  \textsuperscript{1}University of Electronic Science and Technology of China \\
  \textsuperscript{2}Drexel University \\
  \textsuperscript{3}Lehigh University\\
  \textsuperscript{4}The Hong Kong University of Science and Technology (Guangzhou)
  \vspace{3pt}\\
  {\tt\small kong13661@outlook.com \quad xsshi2013@gmail.com \quad kx46@drexel.edu}
}

\begin{document}
\maketitle

 \begin{abstract}

Though diffusion models excel in image generation, their step-by-step denoising leads to slow generation speeds. Consistency training addresses this issue with single-step sampling but often produces lower-quality generations and requires high training costs. In this paper, we show that optimizing consistency training loss minimizes the Wasserstein distance between target and generated distributions. As timestep increases, the upper bound accumulates previous consistency training losses. Therefore, larger batch sizes are needed to reduce both current and accumulated losses. We propose Adversarial Consistency Training (ACT), which directly minimizes the Jensen-Shannon (JS) divergence between distributions at each timestep using a discriminator. Theoretically, ACT enhances generation quality, and convergence. By incorporating a  discriminator into the consistency training framework, our method achieves improved FID scores on CIFAR10 and ImageNet 64$\times$64 and LSUN Cat 256$\times$256 datasets, retains zero-shot image inpainting capabilities, and uses less than $1/6$ of the original batch size and fewer than $1/2$ of the model parameters and training steps compared to the baseline method, this leads to a substantial reduction in resource consumption. Our code is available: \url{https://github.com/kong13661/ACT}

\end{abstract}

\section{Introduction}

Diffusion models, known for their success in image  generation~\cite{ho2020denoising, Score_Based,Estimating_Gradients,yuan2023remind,SecMI,liu2024sora}, utilize diffusion processes to produce high-quality, diverse images. They also perform tasks like zero-shot inpainting \cite{lugmayr2022repaint} and audio generation \cite{Popov2021Grad-TTS,Kong2021DiffWave,kong2023efficient}. However, they have a significant drawback: lengthy sampling times. These models generate target distribution samples by iterative denoising a Gaussian noise input, a process that involves gradual noise reduction until samples match the target distribution. This limitation affects their practicality and efficiency in real-world applications.

The lengthy sampling times of diffusion models have spurred the creation of various strategies to tackle this issue. Several models and techniques have been suggested to enhance the efficiency of diffusion-based image generation \cite{Chen2019Residual, Liu2022Flow, Zheng2023Fast}. Recently, consistency models \cite{Song2023Consistency} have been introduced to speed up the diffusion models' sampling process. A consistency function is one that consistently yields the same output along a specific trajectory. To use consistency models, the trajectory from noise to the target sample must be obtained. By fitting the consistency function, the model can generate data within 1 or 2 steps.

The score-based model \cite{Score_Based}, an extension of the diffusion model in continuous time, gradually samples from a normal distribution $p_T$ to the sample distribution $p_0$. In deterministic sampling, it essentially solves an Ordinary Differential Equation (ODE), with each sample representing an ODE trajectory. Consistency models generate samples using a consistency function that aligns every point on the ODE trajectory with the ODE endpoint. However, deriving the true ODE trajectory is complex. To tackle this, consistency models suggest two methods. The first, consistency distillation, trains a score-based model to obtain the ODE trajectory. The second, consistency training, approximates the trajectory using a conditional one. Compared to distillation, consistency training has a larger error, leading to lower sample quality. The consistency function is trained by equating the model's output at time $t_{n+1}$ with its output at time $t_{n}$.

Generative Adversarial Networks (GANs) \cite{Brock2019Large, Zhang2020Consistency, goodfellow2014generative}, unlike consistency training, can directly minimize the distance between the model's generated and target distributions via the discriminator, independent of the model's output at previous time $t_{n-1}$. Drawing from GANs, we introduce Adversarial Consistency Training. We first theoretically explain the need for large batch sizes in consistency training by showing its equivalence to optimizing the upper bound of the Wasserstein-distance between the model's generated and target distributions. This upper bound consists of the accumulated consistency training loss $\mathcal L^{t_k}_{CT}$, the distance between sampling distributions, and the accumulated error, all of which increase with $t$. Hence, a large batch size is crucial to minimize the error from the previous time $t$. To mitigate the impact of $\mathcal L^{t_k}_{CT}$ and accumulated error, we incorporate the discriminator into consistency training, enabling direct reduction of the JS-divergence between the generated and target distributions at each timestep $t$. Our experiments on CIFAR10 \cite{cifar10}, ImageNet 64$\times$64 \cite{deng2009imagenet} and LSUN Cat 256$\times$256 \cite{yu2015lsun} show that ACT significantly surpasses consistency training while needing less than $1/6$ of the original batch size and less than $1/2$ of the original model parameters and training steps, leading to considerable resource savings. For comparison, we use 1 NVIDIA GeForce RTX 3090 for CIFAR10, 4 NVIDIA A100 GPUs for ImageNet 64$\times$64 and 8 NVIDIA A100 GPUs for LSUN Cat 256$\times$256, while consistency training requires 8, 64, 64 A100 GPUs for CIFAR10, ImageNet 64$\times$64 and LSUN Cat 256$\times$256, respectively.

Our contributions are summarized as follows: 
\begin{itemize}
    \item We demonstrate that consistency training is equivalent to optimizing the upper bound of the W-distance. 
    By analyzing this upper bound, we have identified one reason why consistency training requires a larger batch size.
    
    \item Following our analysis, we propose Adversarial Consistency Training (ACT) to directly optimize the JS divergence between the sampling distribution and the target distribution at each timestep $t$, by incorporating a discriminator into the consistency training process.

    \item Experimental results demonstrate that the proposed ACT significantly outperforms the original consistency training with only less than $1/6$ of the original batch size and less than $1/2$ of the training steps. This leads to a substantial reduction in resource consumption.
    
\end{itemize}

\section{Related works}

\textbf{Generative Adversarial Networks} ~ GANs have achieved tremendous success in various domains, including image generation~\cite{goodfellow2014generative,yuan2020attribute,yuan2023dde} and audio synthesis~\cite{donahue2018adversarial}. However, GAN training faces challenges such as instability and mode collapse, where the generator fails to capture the diversity of the training data. To address these issues, several methods have been proposed. For example, spectral normalization, gradient penalty, and differentiable data augmentation techniques have been developed. Spectral normalization \cite{Miyato2018Spectral} constrains the Lipschitz constant of the discriminator, promoting more stable training. Gradient penalty, as employed in the WGAN-GP \cite{Ishaan2017Improved}, utilizes the gradient penalty to discriminator to limit the range of gradient, so as to avoid the tend of concentrating the weights around extreme values, when using weight clipping in WGAN \cite{arjovsky2017wasserstein}. \cite{Thanh-Tung2019Improving} introduces the concept of zero centered gradient penalty, and StyleGAN2 \cite{Tero2020Analyzing} introduces lazy regularization which performs multiple steps of iteration before computing the gradient penalty to improve the efficiency. Moreover, differentiable data augmentation techniques \cite{Zhao2020Differentiable} have been introduced to enhance the diversity and robustness of GAN models during training. StyleGAN2-ADA \cite{Tero2020Training} improves GAN performance on small datasets by employing adaptive differentiable data augmentation techniques.

\textbf{Diffusion Models} ~ Diffusion models have emerged as highly successful approaches for generating images \cite{Ramesh2022Hierarchical,stable_diffusion}. In contrast to the traditional approach of Generative Adversarial Networks (GANs), which involve a generator and a discriminator, diffusion models generate samples by modeling the inverse process of a diffusion process from Gaussian noise. Diffusion models have shown superior stable training process compared to GANs, effectively addressing issues such as checkerboard artifacts \cite{salimans2016improved, Donahue2017Adversarial,Dumoulin2017Adversarially}.
The diffusion process is defined as follows:
$\boldsymbol{x}_{t}=\sqrt{\alpha_{t}}\boldsymbol{x}_{t-1}+\sqrt{\beta_{t}} \boldsymbol\epsilon_{t},  \boldsymbol\epsilon_{t} \sim \mathcal{N}(\mathbf{0}, \mathbf{I}) $. As $t$ increases, $\beta_t$ gradually increases, causing $\boldsymbol x_t$ to approximate random Gaussian noise. In the reverse diffusion process, $\boldsymbol{x}'_t$ follows a Gaussian distribution, assuming the same variance as in the forward diffusion process. The mean of $\boldsymbol{x}'_t$ is defined as:
$\tilde{\boldsymbol \mu}_{t}=\frac{1}{\sqrt{a_{t}}}\left(\boldsymbol x_{t}-\frac{\beta_{t}}{\sqrt{1-\bar{a}_{t}}} \bar{\boldsymbol{\epsilon}}_{\theta}(\boldsymbol x_t,t)\right)$, 
where $\bar\alpha_t=\prod_{k=0}^t \alpha_k$ and $\bar\alpha_t+\bar\beta_t=1$. The reverse diffusion process becomes:
$\boldsymbol x_{t-1}=\tilde{\boldsymbol \mu}_{t}+\sqrt{\beta_t}\boldsymbol\epsilon,\boldsymbol\epsilon\sim \mathcal{N}(\mathbf{0}, \mathbf{I})$. 
The loss function is defined as $\mathbb{E}_{x_{0}, \bar{\boldsymbol \epsilon}_{t}}\left[\left\|\bar{\boldsymbol \epsilon}_{t}-\boldsymbol \epsilon_{\theta}\left(\sqrt{\bar{\alpha}_{t}} x_{0}+\sqrt{1-\bar{\alpha}_{t}} \bar{\boldsymbol \epsilon}_{t}, t\right)\right\|^{2}\right]. $ 
Score-based models \cite{Score_Based} transforms the discrete-time diffusion process into a continuous-time process and employs Stochastic Differential Equations (SDEs) to express the diffusion process. Moreover, the forward and backward processes are no longer restricted to the diffusion process. They employ the forward process defined as $ d \boldsymbol{x}=\left(\boldsymbol{f}_{t}(\boldsymbol{x})-\frac{1}{2}\left(g_{t}^{2}-\sigma_{t}^{2}\right) \nabla_{\boldsymbol{x}} \log p_{t}(\boldsymbol{x})\right) d t+\sigma_{t} d \boldsymbol{w} $, and the corresponding backward process is $ d \boldsymbol{x}=\left(\boldsymbol{f}_{t}(\boldsymbol{x})-\frac{1}{2}\left(g_{t}^{2}+\sigma_{t}^{2}\right) \nabla_{\boldsymbol{x}} \log p_{t}(\boldsymbol{x})\right) d t+\sigma_{t} d \boldsymbol{\bar w} $, where $\boldsymbol w$ is the forward time Brownian motion and $\boldsymbol {\bar w}$ is the forward time Brownian motion. Compared to GANs, diffusion models have longer sampling time consummations. Several methods have been proposed to accelerate the generation process, including \cite{Salimans2022Progressive,Dockhorn2022Score-Based,Xiao2022Tackling}, DDIM \cite{DDIM}, Consistency models \cite{Song2023Consistency}, etc.

\textbf{Consistency type models} ~ A function is called a consistency function if its output is the same at every point on a trajectory. Formally, given a trajectory, $\boldsymbol x_t, t\in [0, T]$, the function satisfies $f(\boldsymbol x_{t_1})=\mathbb E[f(\boldsymbol x_{t_2})]$, if $t_1, t_2\in [0,T]$. If this trajectory is not a probability trajectory, then the expected symbol $\mathbb E$ in the above formula can be removed. \cite{daras2023consistent} proposed Consistency Diffusion Models (CDM), which proves that when the forward diffusion process satisfies $d \boldsymbol 
x_t = g(t) d\boldsymbol w_t$, $\boldsymbol h(\boldsymbol x,t)=\nabla \log q_t(\boldsymbol x) g^2(t)+\boldsymbol x$ is a consistency function. They add consistency regularity above during training to improve the sampling effectiveness of the model. \cite{Song2023Consistency} proposed consistency models. Unlike consistency diffusion models, Consistency Models (CM) utilize deterministic sampling to obtain a one-step sampling model by learning the mapping from each point $\boldsymbol x_t$ on the trajectory to $\boldsymbol x_0$. When training a diffusion model to obtain the trajectory $\boldsymbol x_t$, it is called consistency distillation. When using conditional-trajectories to approximate non-conditional trajectories, it is called consistency training. Compared to consistency distillation, consistency training has a lower sampling effectiveness. Concurrently, \cite{kim2023consistency} induces a new temporal variable, while calculating the previous step's $x$ through multi-step iteration, and incorporates a discriminator after a period of training and achieved SOTA results in distillation. Our work concentrates on energy-efficient training from scratch also with different objective functions.

\section{Method}

\subsection{Preliminary}
\subsubsection{Score-Based Generative Models}
\label{score_based}

Score-Based Generative Models \cite{Score_Based}, as an extension of diffusion models, extends the diffusion to continuous time, and the forward and backward processes are no longer limited to the diffusion process. Given a distribution $p_t$, where $t\in[0,T]$, $p_0$ is the data distribution and $p_T$ is normal distribution. From $p_0$ to $p_T$, this distribution increasingly approximates a normal distribution. We sample $\boldsymbol x_t$ from $p_t$ distribution. If we can obtain $\boldsymbol x_{t'}$ from the formula $ d \boldsymbol{x}=\left(\boldsymbol{f}_{t}(\boldsymbol{x})-\frac{1}{2}\left(g_{t}^{2}-\sigma_{t}^{2}\right) \nabla_{\boldsymbol{x}} \log p_{t}(\boldsymbol{x})\right) d t+\sigma_{t} d \boldsymbol{w} $, where $\boldsymbol w$ is the forward time Brownian motion and $t'>t$, then we can obtain $\boldsymbol x_{t'}$ from the formula $ d \boldsymbol{x}=\left(\boldsymbol{f}_{t}(\boldsymbol{x})-\frac{1}{2}\left(g_{t}^{2}+\sigma_{t}^{2}\right) \nabla_{\boldsymbol{x}} \log p_{t}(\boldsymbol{x})\right) d t+\sigma_{t} d \boldsymbol{w}$, where $\boldsymbol w$ is the backward time Brownian motion and $t'<t$. If $\sigma_t=0$, this formula turns into a ordinary differential equation $d \boldsymbol{x}=\left(\boldsymbol{f}_{t}(\boldsymbol{x})-\frac{1}{2}g_{t}^{2} \nabla_{\boldsymbol{x}} \log p_{t}(\boldsymbol{x})\right) d t.$
We can generate a new sample by numerically solving this Ordinary Differential Equation (ODE). For each $\boldsymbol x_T\sim p_T$, this ODE describes a trajectory from $\boldsymbol x_T$ to $\boldsymbol x_0$.

\subsubsection{Consistency Training}
\label{sec:consistency_training}
Denote $\{\boldsymbol{x}_t \}$ as a ODE trajectory, a function is called consistency function, if $\boldsymbol g(\boldsymbol x_{t_1},t_1)=\boldsymbol g(\boldsymbol x_{t_2},t_2)$, for any $\boldsymbol x_{t_1}, \boldsymbol x_{t_2}\in\{\boldsymbol{x}_t \}$.
To reduce the time consumption for sampling from diffusion models, consistency training utilizes a model to fit the consistency function $\boldsymbol g(\boldsymbol x_{t_1},t_1)=\boldsymbol g(\boldsymbol x_{t_2},t_2)=\boldsymbol x_0$. The ODE trajectory selected by consistency training is
\begin{equation}
d \boldsymbol x=t\nabla_{\boldsymbol x}\log p_t(\boldsymbol x)dt, t\in [0, T].
    \label{ode}
\end{equation}
In this setting, the distribution of 
$$p_t(\boldsymbol x)=p_0(\boldsymbol x)\ast\mathcal{N}(0,t^2\boldsymbol I),$$
where $\ast$ is convolution operator. The consistency models are denoted as $\boldsymbol f(\boldsymbol x_t,t,\boldsymbol \theta)$. Consistency model is defined as
\begin{small}
\begin{equation}
\boldsymbol f(\boldsymbol x_t,t, \boldsymbol\theta)=\frac{0.5^2}{r_t^2+0.5^2}\boldsymbol x_t+\frac{0.5r_t}{\sqrt{0.5^2+r_t^2}}\boldsymbol F_{\boldsymbol \theta}((\frac{1}{\sqrt{r_t^2+0.5^2}})\boldsymbol x_t,t),
\label{definition_f}
\end{equation}
\end{small}

\noindent where $\boldsymbol \theta$ represents the parameters of the model, $\boldsymbol F_{\boldsymbol\theta}$ is the output of network, $r_t=t-\epsilon$, and $\epsilon$ is a small number for numeric stability.

To train the consistency model $\boldsymbol f(\boldsymbol x_t,t,\theta)$, we need to divide the time interval $[0,T]$ into several discrete time steps, denoted as $t_0=\epsilon<t_1<t_2<\dots<t_N=T$. $N$ gradually increases as the training progresses, satisfying 
\begin{small}
$$N(k)=\lceil \sqrt{\frac{k}{K}((s_1+1)^2-s_0^2)+s_0^2}-1 \rceil+1,$$
\end{small}
where $K$ denotes the total number of training steps, $s_1$ is the end of time steps, $s_0$ is the beginning of time steps and $k$ refers to the current training step.
Denote 
\begin{small}
$$\mathcal L_{CD}^n=\sum_{k=1}^{n}\mathbb E[d(\boldsymbol f(\boldsymbol x_{t_{k}}, t_{k},\boldsymbol \theta), \boldsymbol f(\boldsymbol x_{t_{k-1}}^\Phi, t_{k-1},\boldsymbol\theta^-))],$$
\end{small}
where $d(\cdot)$ is a distance function, $\boldsymbol\theta^-$ is the exponentially moving average of each batch of $\boldsymbol\theta$, and $\boldsymbol x_{t_{n+1}}\sim p_{t_{n+1}}$. $\boldsymbol x_{t_n}^\Phi$ is obtained from $\boldsymbol x_{t_{n+1}}$ through the ODE solver $\Phi$ using \cref{ode}. About $\boldsymbol \theta$ and $\boldsymbol \theta^-$, the equation is given as $\boldsymbol \theta^-_{k+1}=\mu(k) \boldsymbol \theta_k^-+(1-\mu(k))\boldsymbol \theta_k$, where $\mu(k)=\exp(\frac{s_0\log \mu_0}{N(k)})$
and $\mu_0$ is the coefficient at the beginning.

However, calculating $\mathcal L^\Phi_{CD}$ requires training another score-based generative model. They also propose using conditional trajectories to approximate $x_{t_{n}}^\Phi$. This loss is denoted as 
\begin{small}
    $$\mathcal L^{n}_{CT}=\sum_{k=1}^{n}\mathbb E [d(f(\boldsymbol x_0 + t_{k}\boldsymbol z, t_{k},\boldsymbol \theta), f(\boldsymbol x_0 + t_{k-1}\boldsymbol z, t_{k-1},\boldsymbol \theta^-))],$$
\end{small}
where $\boldsymbol x_0\sim p_0$
and $\boldsymbol z\sim \mathcal N(0,I)$. $\mathcal{L}_{CT}=\mathcal{L}^{N}_{CT}$ is called consistency training loss. Using this loss to train the consistency model is called consistency training. This loss is proven \cite{Song2023Consistency} to satisfy 
\begin{small}
\begin{equation}
    \mathcal L^n_{CT}=\mathcal L^n_{CD}+o(\Delta t),
    \label{ct2cd}
\end{equation}
\end{small}
when the ODE solver $\Phi$ is Euler solver.

\subsubsection{Generative Adversarial Networks}

Generative Adversarial Networks (GANs), as generative models, are divided into two parts during training. One part is the generator, denoted as $G(\cdot)$, which is used to generate samples from the approximated target distribution. The other part is the discriminator, denoted as $D(\cdot)$. The training of GANs is alternatively optimizing $G(\cdot)$ and $D(\cdot)$: 1) train to distinguish whether the sample is a generated sample; 2) train $G(\cdot)$ to deceive the discriminator. These two steps are alternated in training. One type of GANs can be described as the following minimax problem: $ \min _{G} \max _{D} V(G, D)=\mathbb{E}_{\boldsymbol x \sim p_{\text {data }}(\boldsymbol x)}[\log D(\boldsymbol x)]+\mathbb{E}_{\boldsymbol z \sim p_{\boldsymbol z}(\boldsymbol z)}[\log (1-D(G(\boldsymbol z)))] $. It can be proven that this minimax problem is equivalent to minimizing the JS-divergence between $p_{\text{data}}$ and $G(\boldsymbol z)$, where $\boldsymbol z \sim p_{\boldsymbol z}$. 

To improve the training stability of GANs, many methods have been proposed. A practical approach is the zero-centered gradient penalty. This is achieved by using the following regularization: 
\begin{small}
\begin{equation}
\mathcal L_{gp}=\Vert \nabla_{\boldsymbol x}D(\boldsymbol x) \Vert^2,\boldsymbol x\sim p_{\text{data}}.
\end{equation}
\end{small}

To reduce computational overhead, this regularization can be applied intermittently every few training steps, rather than at every step.

\subsection{Analysis the Loss Function}
\label{analysis}

\begin{theorem}
If the consistency model satisfies the Lipschitz condition: there exists $L>0$ such that for all $\boldsymbol x$, $\boldsymbol y$ and $t$, we have $\Vert \boldsymbol f(\boldsymbol x,t,\boldsymbol \theta)-\boldsymbol f(\boldsymbol y,t,\boldsymbol \theta)\Vert_2\leq L\Vert \boldsymbol x-\boldsymbol y\Vert_2$, then minimizing the consistency loss will reduce the upper boundary of the W-distance between the two distributions. This can be formally articulated as the following theorem:
\begin{small}
\begin{equation}
\begin{split}
    \mathcal{W}[f_{t_k},g_{t_k}]
    &=\mathcal{W}[f_{t_k},p_0]\\ 
    &\leq L\mathcal{W}[q_{t_k},p_{t_k}]+\mathcal L^{t_k}_{CT}+t_kO(\Delta t)+o(\Delta t),
\end{split}
\label{error}
\end{equation}
\end{small}

where the definition of $p_t$,$\boldsymbol f$, $\mathcal L_{CT}^{t_k}$ and $\boldsymbol g$ is consistent with that in \cref{sec:consistency_training}. $\Delta t=\max (t_k-t_{k-1})$. The distribution $f_t$ is defined as $\boldsymbol f(\boldsymbol x_t,t,\boldsymbol \theta)$, where $\boldsymbol x_t\sim q_t$, and the distribution $g_t$ is defined as $\boldsymbol g(\boldsymbol y_t,t)$, where $\boldsymbol y_t\sim p_t$. The distribution $q_t$ represents the noise distribution when generating samples.
\end{theorem}

\begin{proof}
The W-distance (Wasserstein-distance) is defined as follows:
\begin{small}
$$
\mathcal{W}_\rho[p, q]=\inf_{\gamma \in \prod [p,q]}\iint \gamma(\boldsymbol x, \boldsymbol y)\Vert \boldsymbol x-\boldsymbol y\Vert_\rho d\boldsymbol x d\boldsymbol y,
$$
\end{small}
where $\gamma$ is any joint distribution of $p$ and $q$. For convenience, we take the case of $\rho=2$ and simply denote $\Vert \cdot \Vert$ as $\Vert \cdot \Vert_2$, and denote $\mathcal{W}[p, q]$ as $\mathcal{W}_2[p, q]$. Let $\{\boldsymbol x_{t_k}\}$ or $\{\boldsymbol y_{t_k}\}$ be the points on the same trajectory defined by the ODE in \cref{ode} on the ODE trajectory. For $\mathcal{W}[f_{t_k},g_{t_k}]$, we have the following inequality:
\begin{small}
\begin{align*}
    &\mathcal{W}[f_{t_k},g_{t_k}]\\
    =&\inf_{\gamma^* \in \prod [f_{t_k},g_{t_k}]}\iint \gamma^*(\hat{\boldsymbol  x}_{t_k}, \hat{\boldsymbol y}_{t_k})\Vert \hat{\boldsymbol  x}_{t_k}-\hat{\boldsymbol  y}_{t_k}\Vert_\rho d\hat{\boldsymbol  x}_{t_k} d\hat{\boldsymbol  y}_{t_k} \\
    \overset{(\rom 1)}{\leq}& \iint \gamma(\hat{\boldsymbol  x}_{t_k}, \hat{\boldsymbol  y}_{t_k})\Vert \hat{\boldsymbol  x}_{t_k}-\hat{\boldsymbol  y}_{t_k}\Vert d\hat{\boldsymbol  x}_{t_k} d\hat{\boldsymbol  y}_{t_k}, \gamma \in \prod [f_{t_k},g_{t_k}] \\
    {=}&\mathbb{E}_{\hat{\boldsymbol  x}_{t_k},\hat{\boldsymbol  y}_{t_k}\sim\gamma \in \prod [f_{t_k},g_{t_k}]}[\Vert \hat{\boldsymbol x}_{t_k}-\hat{\boldsymbol  y}_{t_k}\Vert]\\
    \overset{(\rom 2)}{=}&\mathbb{E}_{\boldsymbol x_{t_k},\boldsymbol y_{t_k}\sim\gamma \in \prod [q_{t_k},p_{t_k}]}[\Vert \boldsymbol f(\boldsymbol x_{t_k}, t_k,\phi)-\boldsymbol g(\boldsymbol y_{t_k},t_k)\Vert].
\end{align*}
\end{small}

\noindent Here, (\rom{1}) holds because $\gamma$ is the joint distribution of any $p_t$ and $q_t$. (\rom{2}) is obtained through the law of the unconscious statistician. Since the joint distribution $\gamma \in \prod [q_{t_k},p_{t_k}]$ in the above formula is arbitrary, so we choose the distribution satisfying $\mathbb{E}_{\boldsymbol x_{t_k},\boldsymbol y_{t_k}\sim\gamma^*}[\Vert \boldsymbol y_{t_k} - \boldsymbol x_{t_k}\Vert]=\mathcal{W}[q_{t_k},p_{t_k}]$. We denote it as $\gamma^*$. The expectation $\mathbb{E}_{\boldsymbol x_{t_k},\boldsymbol y_{t_k}\sim\boldsymbol \gamma^*}[\Vert f(\boldsymbol x_{t_k}, t_k,\theta)-g(\boldsymbol y_{t_k},t_k)\Vert]$ satisfies the following inequality:
\begin{small}
\begin{equation}
\begin{split}
    &\mathbb{E}_{\boldsymbol x_{t_k},\boldsymbol y_{t_k}\sim\gamma^*}[\Vert \boldsymbol f(\boldsymbol x_{t_k}, t_k,\boldsymbol \theta)-\boldsymbol g(\boldsymbol y_{t_k},t_k)\Vert]\\
    {\leq}&\mathbb{E}_{\boldsymbol y_{t_k}\sim p_{t_k}}[\Vert \boldsymbol g(\boldsymbol y_{t_k},t_k) - \boldsymbol f(\boldsymbol y_{t_k}, t_k,\boldsymbol \theta)\Vert] + L\mathcal{W}[q_{t_k},p_{t_k}].
    \label{E1}
\end{split}
\end{equation}
\end{small}
If the ODE solver is Euler ODE solver, we have:
\begin{small}
    \begin{equation}
        \begin{split}
    &\mathbb{E}_{\boldsymbol y_{t_k}\sim p_{t_k}}[\Vert \boldsymbol g(\boldsymbol y_{t_k},t_k) - \boldsymbol f(\boldsymbol y_{t_k}, t_k,\boldsymbol \theta)\Vert]\\
    {\leq}&\mathbb{E}_{\boldsymbol y_{t_{k-1}}\sim p_{t_{k-1}}}[\Vert \boldsymbol g(\boldsymbol y_{t_{k-1}},t_{k-1}) - \boldsymbol f(\boldsymbol y_{t_{k-1}}, t_{k-1},\boldsymbol \theta)\Vert]\\
    &\quad+  L(t_{k}-t_{k-1})O(t_{k}-t_{k-1})\\
    &\quad+\mathbb{E}_{\boldsymbol y_{t_{k}}\sim p_{t_{k}}}[\Vert \boldsymbol f(\boldsymbol y_{t_{k-1}}^\phi, t_{k-1},\boldsymbol \theta) - \boldsymbol f(\boldsymbol y_{t_k}, t_k,\boldsymbol \theta)\Vert]\\
    \label{E2}
        \end{split}
    \end{equation}
\end{small}

\noindent The detailed proofs for the aforementioned inequalities can be found in \cref{detail_proof}. We iterate multiple times until $t_0$. At this point, from \cref{definition_f}, we have $\Vert g(y_{t_{0}},t_{0}) - f(y_{t_{0}}, t_{0},\theta)\Vert=0$. So, we can obtain the inequality below:
\begin{small}
\begin{align*}
    &\mathbb{E}_{\boldsymbol y_{t_k}\sim p_{t_k}}[\Vert \boldsymbol g(\boldsymbol y_{t_k},t_k) - \boldsymbol f(\boldsymbol y_{t_k}, t_k,\boldsymbol \theta)\Vert]\\
    \leq&\mathcal L^k_{CD}+\sum_{i=1}^k L(t_i-t_{i-1})O((t_i-t_{i-1}))\\
    \overset{(\rom 1)}{=}&\mathcal L^k_{CT}+\sum_{i=1}^k t_{k}O((\Delta t))+o(\Delta t).
\end{align*}
\end{small}

\noindent Here, (\rom{1}) holds because $\Delta t=\max (t_k-t_{k-1})$, and the relationship between $\mathcal L^k_{CD}$ and $\mathcal L^k_{CT}$ in \cref{ct2cd}. Since consistency function $\boldsymbol g(\boldsymbol x_t,t)=\boldsymbol x_0$, it follows that $\mathcal{W}[f_{t_k},g_{t_k}]=\mathcal{W}[f_{t_k},p_0]$. Putting these together, the proof is complete.
\end{proof}

 Analyzing \cref{error}, $\mathcal{W}[q_{t_k},p_{t_k}]$ is the W-distance between the two sampling distributions, which is independent of the model. We set $q_t = p_t$ to eliminate $\mathcal{W}[q_{t_k},p_{t_k}]$. The term $o(\Delta t)$ and $t_kO(\Delta t)$ originate from approximation errors, where $t_kO(\Delta t)$ increases with the increase of $t_k$. The remaining term is $\mathcal L^{k}_{CT}=\sum_{i=1}^{k}\mathbb E [d(f(\boldsymbol x_0 + t_{i}\boldsymbol z, t_{i},\boldsymbol \theta), f(\boldsymbol x_0 + t_{i-1}\boldsymbol z, t_{i-1},\boldsymbol \theta^-))]$. It can be seen that this term also accumulates errors. The quality of the model's generation depends not only on the current loss at $t_k$, $\mathbb E [d(f(\boldsymbol x_0 + t_{k}\boldsymbol z, t_{k},\boldsymbol \theta), f(\boldsymbol x_0 + t_{k-1}\boldsymbol z, t_{k-1},\boldsymbol \theta^-))]$, but also on the sum of all losses for values less than $k$. These two accumulated errors may be one of the reasons why consistency training requires as large a batch size and large model size as possible. During training, it is not only necessary to ensure a smaller loss at the current $t_k$, but also to use a larger batch size and larger model size to ensure a smaller loss at previous $t$ values. Besides, reducing $\Delta t$ can help to lower this upper bound. However, as described in the original text \cite{Song2023Consistency}, reducing $\Delta t$ in practical applications does not always lead to performance improvements.

\subsection{Enhancing Consistency Training with Discriminator}
\label{enhancing}
Following the analysis in \cref{analysis}, it can be observed that the W-distance at time $t_k$ depends not only on the loss at $t_k$, but also on the loss at previous times. This could be one of the reasons why consistency training requires as large a batch size and model size as possible. However, it can be noted that at each moment $t_k$, the ultimate goal is to reduce the distance between the generated distribution and the target distribution.
In order to reduce the gap between two distributions, we propose not only using the W-distance, but also other distances, such as JS-divergence. Inspired by GANs, we suggest incorporating a discriminator into the training process.

It can be proven that when the generator training loss is given by 
\begin{small}
\begin{equation}
    \mathcal L_G=\log(1-D(\boldsymbol f(\boldsymbol x+t_{n+1}\boldsymbol z,t_{n+1}, \boldsymbol \theta_g), t_{n+1}, \boldsymbol \theta_d)),
    \label{lg}
\end{equation}
\end{small}
and the discriminator training loss is given by 
\begin{small}
\begin{equation}
\begin{split}
    \mathcal L_D=&-\log(1-D(\boldsymbol f(\boldsymbol x_g+t_{n+1}\boldsymbol z,t_{n+1}), \boldsymbol \theta_d)\\
    &-\log(D(\boldsymbol x_r,t_{n+1},\boldsymbol \theta_d)),
\end{split}
\label{ld}
\end{equation}
\end{small}

\noindent minimizing the loss leads to $\min_{\boldsymbol  f} (-2 \log 2 + 2 JSD\left(f_{t_k} \| p_{0}\right))$, which is equivalent to minimizing the JS-divergence. $D$ is the discriminator. It can be observed that this loss does not depend on the previous $t_k$ loss, and can directly optimize the distance between the current $t_k$ distributions. Therefore, the required batch size and model size can be smaller compared to consistency training.

However, although the ultimate goals of the two distances are the same, e.g., when the JS-divergence is $0$, the W-distance is also $0$, at which point the gradient of the discriminator is also $0$. However, at this point, the gradient of $\mathcal L_{CT}$ may not be $0$ due to the aforementioned error. Moreover, when $\mathcal L_{CT}$ is relatively large, the optimization direction of $\mathcal L_{CT}$ may conflict with $\mathcal L_G$. Consider the extreme case where the output of $f_{t_{n}}$ is completely random, it is clear that $\mathcal L_{CT}$ and $\mathcal L_G$ are in conflict, when training $\boldsymbol f$ at time $t_{n+1}$. On the other hand, when $\mathcal L_{CT}$ is relatively small, the model $f$ is easier to fit at $t_{n}$ than at $t_{n+1}$, thus generating better quality. Also, since $x_t$ and $x_{t_{n+1}}$ are close enough, their discriminators are also close enough, thus jointly improving the generation quality. 
Therefore, we employ the coefficient $\lambda$ to balance the proportion between $\mathcal L_{CT}$ and $\mathcal L_G$. Furthermore, as $\mathcal L^k_{CT}$ increases with $k$, the W-distance also increases. In order to improve the performance of consistency training, the weight of $\mathcal L_G$ should also increase. We utilize the formula \cref{lambda} to give $\mathcal L_G$ more weight, where $w$ is the weight at $n=N-1$, and $w_{mid}$ is the weight at $n=(N-1)/2$.
\begin{small}
\begin{equation}
    \lambda_N (n)=w\left(\frac{n}{N-1}\right)^{\log_{\frac{1}{2}}(\frac{w_{mid}}{w})} \label{lambda}.
\end{equation}
\end{small}

\noindent Please note, even though the fitting targets of all $f_{t_k}$ are $q_0$, we choose for the form $D(\boldsymbol x_t,t,\boldsymbol \theta_d)$ rather than $D(\boldsymbol x_t,\boldsymbol \theta_d)$ when constructing the discriminator. Although theoretically, the optimal distribution of the generator trained by these two discriminators is $p_0$, and for two similar samples, the discriminator in the form of $D(\boldsymbol x_t,\boldsymbol \theta_d)$ will generate similar gradients at different $t$, we find in our experiments \cref{ablation_discriminator} that this form of discriminator is not as effective as $D(\boldsymbol x_t,t,\boldsymbol \theta_d)$. The training algorithm is described in \cref{algorithm1}.

\subsection{Gradient Penalty Based Adaptive Data Augmentation}
\label{actaug}
For smaller datasets, in the field of GANs, there are many data augmentation works to improve generation effects. Inspired by StyleGAN2-ADA\cite{Tero2020Training}, we also utilize adaptive differentiable data augmentation. However, unlike StyleGAN2-ADA, which adjusts the probability of data augmentation based on the accuracy of the discriminator over time, it is difficult to adjust the augmentation probability through the accuracy of a single discriminator in our model due to the varying training difficulties at different $t$. As described in \cref{gp_stablility}, we find that the stability of the discriminator's gradient has a significant impact on training. This may be due to the interaction between $\mathcal L_{CT}$ and $\mathcal L_G$. We propose to adjust the probability of data augmentation based on the value of the gradient penalty over time. Given a differential data augmentation function $A(\boldsymbol x, p_{aug})$, where $p_{aug}$ is the probability of applying the data augmentation, the augmented discriminator is defined by:
$$
D_{aug}(\boldsymbol x_t,t, p_{aug},\boldsymbol \theta_d)=D(A(\boldsymbol x_t, p_{aug}), t, \boldsymbol \theta_d).
$$
The probability $p_{aug}$ is updated by
$$
p_{aug}\leftarrow \text{Clip}_{[0, 1]}(p_{aug}+2([\mathcal L_{gp}^- \geq \tau] - 0.5)p_r),
$$
where $[\cdot]$ denotes the indicator function, which takes a value of $1$ when the condition is true and $0$ otherwise. $\text{Clip}_{[0,1]}(\cdot)$ represents the operation of clipping the value to the interval $[0,1]$. $p_r$ denotes the update rate at each iteration, and $\mathcal L_{gp}^-$ is the exponential moving average of $\mathcal L_{gp}$, defined as $\mathcal L_{gp}^-=\mu_p\mathcal L_{gp}^-+(1-\mu_p)\mathcal L_{gp}$. $p_r$ and $\mu_p$ are constants within the range $[0,1]$. This algorithm is described in \cref{act_aug} shown in \cref{sec:act}. Our motivation for proposing the use of data augmentation is to mitigate the overfitting phenomenon in the discriminator. We conduct experiments on CIFAR10 to verify the method. However, the performance of data augmentation on large datasets, such as ImageNet 64$\times$64, remains to be explored.

\begin{table}[]
\centering
\caption{Training steps and model parameter size are reported. BS stands for Batch Size. For ACT, Params represent parameters of the consistency model + discriminator.}
\label{resource}
    \vspace{-2mm}
\adjustbox{width=0.87\columnwidth}{
\begin{tabular}{lllllll}
\hline
 Dataset                  & Method  & BS   & Steps & Params          & Fid  \\ \hline
\multirow{5}{*}{CIFAR10}  & CT      & 512  & 800K  & 73.9M           & 8.7  \\
                          & CT      & 256  & 800K  & 73.9M           & 10.4 \\
                          & CT      & 128  & 800K  & 73.9M           & 14.4 \\
                          & ACT-Aug & \textbf{80}   & \textbf{300K}  & \textbf{27.5M+14.1M}     & \textbf{6.0}  \\ \hline
\multirow{2}{*}{ImageNet} & CT      & 2048 & 800K  & 282M            & 13.0 \\
                          & ACT     & \textbf{320}  & \textbf{400K}  & \textbf{107M+54M}        & \textbf{10.6} \\ \hline
\multirow{2}{*}{LSUN Cat} & CT      & 2048 & 1000K & 458M            & 20.7  \\
                          & ACT     & \textbf{320}  & \textbf{165K}  & \textbf{113M+57M}        & \textbf{13.0}  \\ \hline
\end{tabular}
}
    \vspace{-1mm}
\end{table}

\begin{table}[]
\centering
\caption{Sample quality of ACT on the ImageNet dataset with the resolution of $64\times 64$. Our ACT significantly outperforms CT.}
    \vspace{-2mm}
\label{imagenet64}
\adjustbox{width=0.82\columnwidth}{
\begin{tabular}{lllll}
\hline
Method      & NFE ($\downarrow$)& FID ($\downarrow$) & Prec. ($\uparrow$) & Rec. ($\uparrow$) \\ \hline
BigGAN-deep \cite{Brock2019Large} & 1   & 4.06 & \textbf{0.79}  & 0.48 \\
ADM   \cite{dhariwal2021diffusion}      & 250 & \textbf{2.07} & 0.74  & 0.63 \\
EDM    \cite{Karras2022Elucidating}     & 79  & 2.44 & 0.71  & \textbf{0.67} \\
DDPM  \cite{ho2020denoising}      & 250 & 11.0 & 0.67  & 0.58 \\
DDIM  \cite{DDIM}      & 50  & 13.7 & 0.65  & 0.56 \\
DDIM  \cite{DDIM}      & 10  & 18.3 & 0.60  & 0.49 \\ \hline
CT          & 1   & 13.0 & \textbf{0.71}  & 0.47 \\
ACT         & 1   & \textbf{10.6} &   0.67    &   \textbf{0.56}   \\ \hline
\end{tabular}
}
    \vspace{-4mm}
\end{table}

\begin{algorithm}[tbp]
\caption{Adversarial Consistency Training}
\label{algorithm1}
\begin{algorithmic}[1]
\State \textbf{Input:} dataset $\mathcal{D}$, initial consistency model parameter $\theta_g$, discriminator $\theta_d$, step schedule $N(\cdot)$, EMA decay rate schedule $\mu(\cdot)$, optimizer $\text{opt}(\cdot,\cdot)$, discriminator $D(\cdot,\cdot,\theta_d)$, adversarial rate schedule $\lambda(\cdot)$, gradient penalty weight $w_{gp}$, gradient penalty interval $I_{gp}$.
\State $\boldsymbol \theta_g^-\leftarrow \boldsymbol \theta$ and $k\leftarrow 0$
\Repeat
\State Sample $\boldsymbol x\sim \mathcal{D}$, and $n\sim\mathcal{U}[\![1,N(k)]\!]$
\State Sample $\boldsymbol z\sim \mathcal{N}(0,\boldsymbol I)$ \Comment{Train Consistency Model}
\State $\mathcal{L}_{CT}\leftarrow$
\Statex \hspace{\algorithmicindent} \quad $d(\boldsymbol f(\boldsymbol x+t_{n+1}\boldsymbol z,t_{n+1}, \boldsymbol \theta_g), \boldsymbol f(\boldsymbol x+t_n\boldsymbol z,t_n,\boldsymbol \theta_g^-))$
\State $\mathcal L_G\leftarrow$
\Statex \hspace{\algorithmicindent} \quad $\log(1-D(\boldsymbol f(\boldsymbol x+t_{n+1}\boldsymbol z,t_{n+1}, \boldsymbol \theta_g), t_{n+1}, \boldsymbol \theta_d))$
\State $\mathcal{L}_f\leftarrow(1-\lambda_{N(k)}(n+1))\mathcal{L}_{CT} + \lambda_{N(k)}(n+1)\mathcal L_G$
\State $\boldsymbol \theta_g\leftarrow\text{opt}(\boldsymbol \theta_g,\nabla_{\boldsymbol \theta_g}(\mathcal{L}_f))$
\State $\boldsymbol \theta_g^-\leftarrow\text{stopgrad}(\mu(k)\boldsymbol \theta_g^-+(1-\mu(k))\boldsymbol \theta_g)$

\Statex

\State Sample $\boldsymbol x_g\sim \mathcal{D}$, $\boldsymbol x_r\sim \mathcal{D}$, and $n\sim\mathcal{U}[\![1,N(k)]\!]$
\State Sample $\boldsymbol z\sim \mathcal{N}(0,\boldsymbol I)$\Comment{Train Discriminator}
\State $\mathcal L_D\leftarrow-\log(D(\boldsymbol x_r,t_{n+1},\boldsymbol \theta_d))$
\Statex \hspace{\algorithmicindent} \quad $-\log(1-D(\boldsymbol f(\boldsymbol x_g+t_{n+1}\boldsymbol z,t_{n+1},\boldsymbol \theta_d))$
\State $\mathcal L_{gp}\leftarrow$
\Statex \hspace{\algorithmicindent} \quad $w_{gp}\Vert\nabla_{\boldsymbol x_r} D(\boldsymbol x_r,t_{n+1},\boldsymbol \theta_d)\Vert^2[k \mod I_{gp}=0]$
\State $\mathcal{L}_d\leftarrow \lambda_{N(k)}(n+1)\mathcal{L}_D + \lambda_{N(k)}(n+1)\mathcal{L}_{gp}$
\State $\boldsymbol \theta_d\leftarrow\text{opt}(\boldsymbol \theta_d,\nabla_{\boldsymbol \theta_d}(\mathcal{L}_d))$
\State $k\leftarrow k+1$
\Until{convergence}
\end{algorithmic}
\end{algorithm}

\section{Experiments}

In this section, we report experimental settings and results on CIFAR-10, ImageNet64 and LSUN Cat 256 datasets.

\subsection{Generation Performance}
In this section, we report the performance of our model on the CIFAR10, ImageNet 64$\times$64 datasets and LSUN Cat 256$\times$256 datasets. The results demonstrate a significant improvement of our method over the original approach. We exhibit the results on CIFAR10 in \cref{cifar10-sample}, on ImageNet 64$\times$64 in \cref{imagenet64} and on LSUN Cat 256$\times$256 in \cref{lsun_cat}, respectively. The FID on CIFAR10 improves from 8.7 to 6.0. It improves from 13 to 10.6 on ImageNet 64$\times$64, and it improves from 20.7 to 13.0 on LSUN Cat 256$\times$256.

Furthermore, we demonstrate the performance of the consistency training on different batch sizes, and the sizes of the models used by the proposed method and consistency training, in \cref{resource}. As can be discerned from the data in the table, the batch size has a significant impact on consistency training. When the batch size is set to 256, the FID score escalates to 10.4 from 8.7. Besides, with a batch size of 128, the FID rises to 14.4. On the CIFAR10 dataset, the proposed method outperforms consistency training, achieving an FID of 6.0 with a batch size of 80, versus 8.7 with a batch size of 512. On ImageNet 64x64, it achieves an FID of 10.6 with a batch size of 320, compared to consistency training's 13.0 with a batch size of 2048. Besides, on LSUN Cat 256 $\times$ 256, the proposed method attains an FID of 13.0 with a batch size of 320, better than consistency training's 20.7 with a batch size of 2048. \cref{Generated_samples} shows the generated samples from model training on ImageNet 64$\times$64 and LSUN Cat 256$\times$256. \cref{cat_1,cat_2} shows more generated samples from model training on LSUN Cat 256$\times$256. \cref{metrics} provides explanations for all metrics. \cref{inpainting_sec} shows zero-shot image inpainting.

\subsection{Resource Consumption}
We utilize the DDPM model architecture as our backbone. While DDPM's performance isn't as high as \cite{dhariwal2021diffusion} and \cite{Score_Based}, it has fewer parameters and attention layers, enabling faster execution. Our model is significantly smaller than the 63.8M model used by consistency training on CIFAR10, with only 27.5M (41.6M with discriminator during training) parameters. On the ImageNet 64$\times$64 dataset, our model, with only 107M parameters (161M with discriminator during training), is smaller than the 282M model used by consistency training. The smaller model and batch size reduce resource consumption. In our experiments on CIFAR10, we utilize 1 NVIDIA GeForce RTX 3090, as opposed to the 8 NVIDIA A100 GPUs used for consistency training. For the ImageNet 64$\times$64 experiments, we employ 4 NVIDIA A100 GPUs, in contrast to the 64 A100 GPUs used for training in the consistency training setup. For the LSUN Cat 256$\times$256 experiments, we employ 8 NVIDIA A100 GPUs, in contrast to the 64 A100 GPUs used for training in the consistency training setup \cite{Song2023Consistency}. 

\begin{table}[]
\centering
\caption{Sample quality of ACT on the CIFAR10 dataset. We compare ACT with state-of-the-art GANs and (efficient) diffusion models. We show that ACT achieves the best FID and IS among all the one-step diffusion models.
}
    \vspace{-1mm}
\label{cifar10-sample}
\adjustbox{width=0.8\columnwidth}{
\begin{tabular}{llll}
\hline
Method           & NFE ($\downarrow$)  & FID ($\downarrow$)  & IS ($\uparrow$)  \\ \hline
BigGAN \cite{Brock2019Large}          & 1    & 14.7 & 9.22 \\
AutoGAN \cite{Gong2019Autogan}         & 1    & 12.4 & 8.40 \\
ViTGAN  \cite{Lee2022ViTGAN}         & 1    & 6.66 & 9.30 \\
TransGAN \cite{jiang2021transgan}        & 1    & 9.26 & 9.05 \\
StyleGAN2-ADA \cite{Tero2020Training}   & 1    & 2.92 & \textbf{9.83} \\
StyleGAN2-XL \cite{Sauer2022StyleGAN-XL}    & 1    & \textbf{1.85} &  -    \\ \hline
Score SDE \cite{Score_Based}       & 2000 & 2.20 & \textbf{9.89} \\
DDPM \cite{ho2020denoising}            & 1000 & 3.17 & 9.46 \\
EDM  \cite{Karras2022Elucidating}            & 36   & \textbf{2.04} & 9.84 \\
DDIM  \cite{DDIM}           & 50   & 4.67 &  -    \\
DDIM \cite{DDIM}            & 20   & 6.84 &  -    \\
DDIM \cite{DDIM}            & 10   & 8.23 &  -    \\ \hline
1-Rectified Flow \cite{Liu2023Flow} & 1    & 378  & 1.13 \\
Glow  \cite{Kingma2018Glow}           & 1    & 48.9 & 3.92 \\
Residual FLow \cite{Chen2019Residual}   & 1    & 46.4 &   -   \\
DenseFlow  \cite{Grcić2021Densely}      & 1    & 34.9 &   -   \\
DC-VAE  \cite{Parmar2021Dual}         & 1    & 17.9 & 8.20 \\
CT  \cite{Song2023Consistency}             & 1    & 8.70 & 8.49 \\
ACT              & 1    & 6.4  &   8.93   \\
ACT-Aug          & 1    & \textbf{6.0}  &   \textbf{9.15}   \\ \hline
\end{tabular}
}
    \vspace{-3mm}
\end{table}

    \vspace{-1mm}

\subsection{Ablation Study}
    \vspace{-1mm}
\label{ablation_study}
\subsubsection{Impacts of $\lambda_N$}
    \vspace{-1mm}

When $\lambda_N\equiv 0$, this reduces to consistency training. Conversely, when $\lambda_N\equiv 1$, it becomes Generative Adversarial Networks (GANs). According to the analysis in \cref{analysis}, as $\lambda_N$ increases, adversarial consistency training gains the capacity to enhance model performance with smaller batch sizes, leveraging the discriminator. However, as discussed in \cref{enhancing}, an overly large $\lambda_N$ can lead to an excessive consistency training loss, thereby causing a conflict between $\mathcal L_{CT}$ and $\mathcal L_G$. Furthermore, it has been noted in the literature that for GANs, high-dimensional inputs may detrimentally affect model performance \cite{padala2021effect}. Therefore, as $\lambda_N$ increases, the model performance exhibits a pattern of initial improvement followed by a decline. Firstly, we demonstrate the phenomenon of mode collapse when $\lambda_N\approx 1$ on CIFAR10. As illustrated in \cref{mode_collapse}, the phenomenon of mode collapse is observed. It can be noted that, apart from the initial $t_k$ where the residual structure from  \cref{definition_f} results in outputs with substantial input components, preventing mode collapse, the other $t_k$ values all exhibit mode collapse. 

For a score-based model as defined in \cref{score_based}, the learned sampling process is the reverse of the diffusion process $p_{t}(\boldsymbol x_0|\boldsymbol x_t)$. 
However, the distribution $q_{t}(\boldsymbol x_0|\boldsymbol x_t)$ learned via \cref{lg,ld} does not consider the forward process of the diffusion. We conduct further experiments where the form of the discriminator is changed to $D(\boldsymbol x_0,\boldsymbol x_t,t,\boldsymbol\theta_d)$, and it can be proven \cref{condition_discriminator} that the distribution learned by the generator is $p_{t}(\boldsymbol x_0|\boldsymbol x_t)$. However, we also observe the phenomenon of mode collapse in our experiments. \cref{imagenet_crash} illustrates the training collapse on ImageNet 64$\times$64 when $\lambda_N\equiv 0.3$. It can be observed that at around 150k training steps, the $\mathcal L_{CT}$ becomes unstable and completely collapses around 170k. We have included the training curves for the proper $\lambda_N$ in the \cref{proper_lambda}. It can be observed that at this point, $\mathcal L_{CT}$ and several other training losses remain stable. Essentially, a smaller $w_{mid}$ and a larger $w$ are preferable choices.

\subsubsection{Connection between gradient penalty and training stability}
\label{gp_stablility}
In \cref{enhancing}, we analyze the relationship between $\mathcal L_{CT}$ and $\mathcal L_{G}$, highlighting the importance of gradient stability. In this section, we conduct experiments to validate our previous analysis and demonstrate the rationality of the ACT-Aug method proposed in \cref{actaug}.

\cref{imagenet_crash} illustrates the relationship among the values of the gradient penalty ($\mathcal L_{gp}$), consistency training loss ($\mathcal L_{CT}$), and FID. It can be observed that almost every instance of instability in $\mathcal L_{CT}$ is accompanied by a relatively large $\mathcal L_{gp}$. \cref{cifar10_noada} illustrates the relationship among these three on the CIFAR10 dataset. It can be seen that in the mid-stage of training, $\mathcal L_{gp}$ begins to slowly increase, a process that is accompanied by a gradual increase in $\mathcal L_{CT}$ and FID. Therefore, we believe that gradient stability is crucial for adversarial consistency training. Based on this, we propose ACT-Aug (\cref{actaug}) on small datasets, using $\mathcal L_{gp}$ as an indicator to adjust the probability of data augmentation, thereby stabilizing $\mathcal L_{gp}$ around a certain value.

\begin{table}[]
\centering
\caption{Sample quality of ACT on the LSUN Cat dataset with the resolution of 256$\times$256. Our ACT significantly outperforms CT. $^\dag$Distillation techniques.}
    
    \vspace{-2mm}
\label{lsun_cat}
\adjustbox{width=0.85\columnwidth}{
\begin{tabular}{lllll}
\hline
Method      & NFE ($\downarrow$)& FID ($\downarrow$) & Prec. ($\uparrow$) & Rec. ($\uparrow$) \\ \hline
DDPM  \cite{ho2020denoising} & 1000   & 17.1 & 0.53  & 0.48 \\
ADM   \cite{dhariwal2021diffusion}     & 1000 & \textbf{5.57} & 0.63  & \textbf{0.52} \\
EDM  \cite{Karras2022Elucidating}    & 79  & 6.69 & \textbf{0.70}  & 0.43 \\
PD$^\dag$
\cite{Salimans2022Progressive}    
& 1  & 18.3 & 0.60  & 0.49 \\
CD$^\dag$ \cite{Song2023Consistency}       & 1    &   11.0   &     0.65           &   0.36   \\ \hline
CT \cite{Song2023Consistency}       & 1   & 20.7 & 0.56  & 0.23 \\
ACT         & 1   & \textbf{13.0} &   \textbf{0.69}    &   \textbf{0.30}   \\ \hline
\end{tabular}
}
    \vspace{-4mm}
\end{table}

\begin{figure}
    \centering
    \includegraphics[width=0.98\columnwidth]{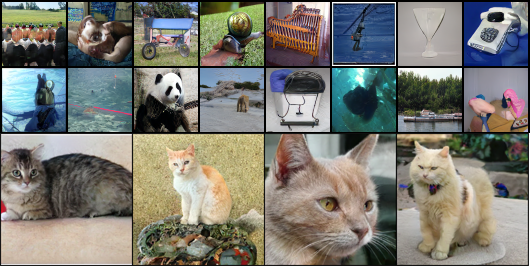}
    \vspace{-1mm}
    \caption{Generated samples on ImageNet 64$\times$64 (top two rows) and LSUN Cat 256$\times$256 (the third row).}
    \vspace{-5mm}
    \label{Generated_samples}
\end{figure}

\begin{figure}
    \centering
    \includegraphics[width=0.95\columnwidth]{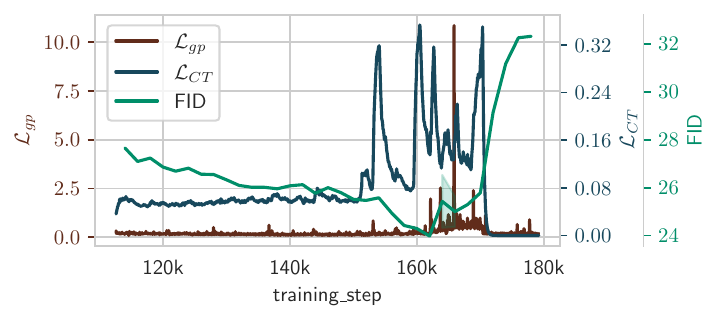}
    \vspace{-4mm}
    \caption{$\mathcal L_{gp}$, $\mathcal{L}_{CT}$, and FID of ACT on ImageNet 64x64 ($\lambda_N\equiv 0.3$, an overly large $\lambda_N$ leads to training collapse. Additionally, drastic changes in $\mathcal L_{gp}$ closely follow changes in $\mathcal L_{CT}$).}
    \vspace{-4mm}
    \label{imagenet_crash}
\end{figure}

\begin{figure}
    \centering
    \includegraphics[width=0.95\columnwidth]{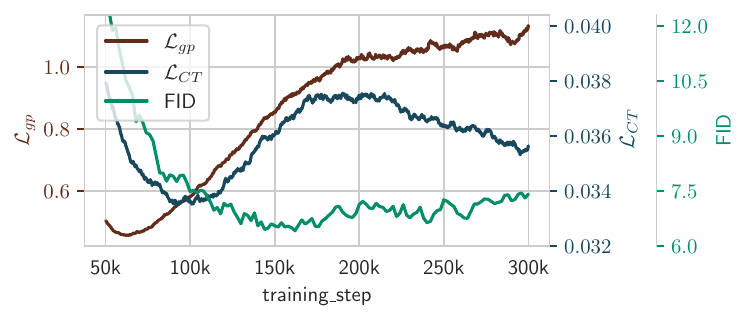}
    \vspace{-4mm}
    \caption{$\mathcal L_{gp}$, $\mathcal{L}_{CT}$, and FID of ACT on CIFAR10 ($\lambda_N\equiv 0.3$, an appropriate $\lambda_N$. In the later stages of training, without data augmentation, $\mathcal L_{CT}$, $\mathcal L_{gp}$, and FID all show relatively large increases).}
    \label{cifar10_noada}
    \vspace{-5mm}
\end{figure}

\subsubsection{Discriminator}
\vspace{-1mm}
\label{ablation_discriminator}
\textbf{Activation Function} Generally, GANs employ LeakyReLU as the activation function for the discriminator. This function is typically considered to provide better gradients for the generator. On the other hand, SiLU is the activation function chosen for DDPM, and it is generally regarded as a stronger activation function compared to LeakyReLU.
\cref{discriminator} displays the FID scores of different activation functions on CIFAR10 at 50k and 150k training steps. Contrary to previous findings, we discovery that utilizing the SiLU function for the discriminator leads to faster convergence rates and improved final performance. A possible reason is that $\mathcal L_{CT}$ provides an additional gradient direction, which mitigates the overfitting of the discriminator.

\noindent\textbf{Different Backbone} ~ \cref{discriminator} also displays the FID scores of different architecture on CIFAR10 at 50k and 150k training steps. In our investigation, we have evaluated the discriminators of StyleGAN2, ProjectedGAN and the downsampling part of DDPM (simply denoted as DDPM) as described in \cref{arch}. Due to the significant role of residual structures in designing GANs' discriminators, we incorporate residual connections between different downsampling blocks in DDPM, denoted as DDPM-res. It can be observed that DDPM performs the best. Although DDPM-res exhibits a faster convergence rate during the early stages of training, its performance in the later stages is not as satisfactory as that of DDPM. Furthermore, we find that DDPM demonstrates superior training stability compared to DDPM-res. We also experiment with whether or not to feed $t$ into the discriminator, denoted as $t$-emb. We find that feeding $t$ yields better results. This might be due to the fact that the optimal value of the discriminator varies with different $t_k$, hence the necessity of $t$-emb for better fitting.

\begin{table}[]
\centering
\caption{Ablation study of the discriminator.}
    \vspace{-1mm}
\label{discriminator}
\adjustbox{width=0.9\columnwidth}{
\begin{tabular}{lllll}
\hline
Discriminator & Activation  & $t$-emb & Fid (50k) & Fid (150k) \\ \hline
DDPM-res & LeakyReLU & False  &   18.7  &  10.6 \\
DDPM-res & LeakyReLU & True  &  11.5   & 7.4 \\
DDPM-res & SiLU & True  &    \textbf{9.9}  & 7.0 \\
DDPM     & SiLU        & True  & 12.5 & \textbf{6.5} \\
StyleGAN2     & LeakyReLU & True  & 16.7 & 9.5 \\
ProjectedGAN  & LeakyReLU & True  & 19.4 & 16.6 \\ \hline
\end{tabular}
}
    \vspace{-5mm}
\end{table}

\section{Conclusion}
\vspace{-1mm}

We proposed Adversarial Consistency Training (ACT), an improvement over consistency training. Through analyzing the consistency training loss, which is proven to be the upper bound of the W-distance between the sampling and target distributions, we introduced a method that directly employs Jensen-Shannon Divergence to minimize the distance between the generated and target distributions. This approach enables superior generation quality with less than $1/6$ of the original batch size and approximately $1/2$ of the original model parameters and training steps, thereby having smaller resource consumption. Our method retains the beneficial capabilities of consistency models, such as inpainting. Additionally, we proposed to use gradient penalty-based adaptive data augmentation to improve the performance on small datasets. The effectiveness has been validated on CIFAR10, ImageNet 64$\times$64 and LSUN Cat 256$\times$256 datasets, highlighting its potential for broader application in the field of image generation. 

However, the interaction between $\mathcal L_{CT}$ and $\mathcal L_{G}$ can be further explored to improve our method. In addition to using JS-Divergence, other distances can also be used to reduce the distance between the generated and target distributions. In the future, we will focus on these two aspects to further boost the performance.

\vspace{-2mm}
\section{Acknowledgement}
\vspace{-1mm}
Fei Kong and Xiaoshuang Shi were supported by the National Natural Science Foundation of China (No. 62276052).

\clearpage
\clearpage
\setcounter{page}{1}
\maketitlesupplementary
\appendix

\counterwithin{figure}{section}
\counterwithin{table}{section}
\renewcommand{\thefigure}{\thesection\arabic{figure}}
\renewcommand{\thetable}{\thesection\arabic{table}}

\section{Architecture and Experiment settings}
\label{arch}

\textbf{Architecture} ~ For the consistency model architecture, we employ a structure similar to that of DDPM \cite{ho2020denoising}, with the exception of altering the corresponding embeddings to continuous time. We utilize the Python library diffusers \cite{von-platen-etal-2022-diffusers}. In terms of the discriminator, we employ the downsampling structure in the DDPM, preserving it up to the mid-block. Subsequently, a linear layer is added to map it to  $\mathbb R$. Additionally, the layers-per-block parameter is set to 150\% of that in the consistency model, with all other parameters remaining the same. The parameters passed to the UNet2DModel are listed in \cref{unet}. B=128. In the context of block type, `D' represents DownBlock2D, `A' stands for either AttnDownBlock2D or AttnUpBlock2D, and `U' means UpBlock2D.

\begin{table}[h!]
\centering
\adjustbox{width=1\columnwidth}{
\begin{tabular}{l|l|l|l}
\hline
                     & CIFAR10       & ImageNet 64$\times$64 & LSUN Cat 256$\times$256  \\ \hline
layers\_per\_block   & 2             & 2                     & 2                     \\
block\_out\_channels & (1B,1B,2B,2B) & (1B,2B,2B,4B,4B)      & (1B,1B,2B,2B,4B,4B)      \\
down\_block\_types   & DADD          & DDADD                 & DDDDAD                \\
up\_block\_types     & UUAU          & UUAUU                 & UAUUUU                 \\
attention\_head\_dim & 8             & 16                    & 16                    \\ \hline
\end{tabular}
}
\caption{The parameters passed to the UNet2DModel. For those not listed, the default settings from the diffusers library are used.}
\label{unet}
\end{table}

\noindent\textbf{Experiment settings} ~ In this section, we report the configuration of various hyperparameters within our experimental framework.  \cref{experimental_settings} provides a summary of the experimental setup. Unless otherwise specified, the learning rate for both the consistency model and the discriminator is identical. The experiments conducted during the ablation study (\cref{ablation_study}), maintain consistency with the settings outlined in  this table, with the exception of the parameters specifically varied for the ablation study. Additionally, when employing the ProjectedGAN as the discriminator, the learning rate of discriminator is set to $0.002$, with $w$ and $w_{mid}$ values at $0.1$.

\label{metrics}
\noindent\textbf{Metrics} ~ The metrics used are IS, FID, Improved Precision and Improved Recall. The Inception Score (IS), introduced in \cite{salimans2016improved}, assesses a model's ability to generate convincing images of distinct ImageNet classes and capture the overall class distribution. However, it has a limitation in that it doesn't incentivize capturing the full distribution or the diversity within classes, leading to models with high IS even if they only memorize a small portion of the dataset, as noted in \cite{Barratt2018A}. To address the need for a metric that better reflects diversity, the Fréchet Inception Distance (FID) was introduced in \cite{Martin2017GANs}. This metric is argued to align more closely with human judgment than IS, and it quantifies the similarity between two image distributions in the latent space of Inception-V3 as detailed in \cite{szegedy2016Rethinking}. Additionally, \cite{nniemi2019Improved} developed Improved Precision and Recall metrics that evaluate the fidelity of generated samples by determining the proportion that aligns with the data manifold (precision) and the diversity by the proportion of real samples that are represented in the generated sample manifold (recall).

\begin{table}[]
\centering
\adjustbox{width=0.9\columnwidth}{
\begin{tabular}{l|l|l|l}
\hline
Hyperparameter      & CIFAR10           & ImageNet      & LSUN Cat       \\
                    &                   & 64$\times$64  & 256$\times$256 \\ \hline
Discriminator       & DDPM              & DDPM          & DDPM           \\
Learning rate       & 1e-4              & 5e-5          & 1e-5           \\
Batch size          & 80                & 320           & 320            \\
$\mu_0$             & 0.9               & 0.95          & 0.95           \\
$s_0$               & 2                 & 2             & 2              \\
$s_1$               & 150               & 200           & 150            \\
$w_{mid}$                 & 0.3               & 0.2           & 0.1            \\
$w$           & 0.3               & 0.6           & 0.6            \\
$I_{gp}$            & 16                & 16            & 16             \\
$w_{gp}$            & 10                & 10            & 10             \\
$\tau$              & 0.55              & -             & -              \\
$\mu_p$             & 0.93              & -             & -              \\
$p_r$               & 0.05              & -             & -              \\
Training iterations & 300k              & 400k          & 165k           \\
Mixed-Precision     & No                & Yes           & Yes            \\
Number of GPUs      & 1$\times$RTX 3090 & 4$\times$A100 & 8$\times$A100  \\ \hline
\end{tabular}
}
\caption{Summary of the experimental setup.}
\label{experimental_settings}
\end{table}

\section{Details of the Proof for Theorem 3.1}
\label{detail_proof}
Details for \cref{E1}:
\begin{align*}
    &\mathbb{E}_{\boldsymbol x_{t_k},\boldsymbol y_{t_k}\sim\gamma^*}[\Vert \boldsymbol f(\boldsymbol x_{t_k}, t_k,\boldsymbol \theta)-\boldsymbol g(\boldsymbol y_{t_k},t_k)\Vert]\\
    =&\mathbb{E}_{\boldsymbol x_{t_k},\boldsymbol y_{t_k}\sim\gamma^*}[\Vert \boldsymbol g(\boldsymbol y_{t_k},t_k) - \boldsymbol f(\boldsymbol y_{t_k}, t_k,\boldsymbol \theta) \\
    &\qquad \qquad \qquad+ \boldsymbol f(\boldsymbol y_{t_k}, t_k,\boldsymbol \theta) - \boldsymbol f(\boldsymbol x_{t_k}, t_k,\boldsymbol \theta)\Vert]\\
    \leq&\mathbb{E}_{\boldsymbol x_{t_k},\boldsymbol y_{t_k}\sim\gamma^*}[\Vert \boldsymbol g(\boldsymbol y_{t_k},t_k) - \boldsymbol f(\boldsymbol y_{t_k}, t_k,\boldsymbol \theta)\Vert \\
    &\qquad \qquad \qquad+ \Vert \boldsymbol f(\boldsymbol y_{t_k}, t_k,\boldsymbol \theta) - \boldsymbol f(\boldsymbol x_{t_k}, t_k,\boldsymbol \theta)\Vert]\\
    \overset{(\rom 1)}{\leq}&\mathbb{E}_{\boldsymbol x_{t_k},\boldsymbol y_{t_k}\sim\gamma^*}[\Vert \boldsymbol g(\boldsymbol y_{t_k},t_k) - \boldsymbol f(\boldsymbol y_{t_k}, t_k,\boldsymbol \theta)\Vert \\
    &\qquad \qquad \qquad+ L\Vert \boldsymbol y_{t_k} - \boldsymbol x_{t_k}\Vert]\\
    =&\mathbb{E}_{\boldsymbol x_{t_k},\boldsymbol y_{t_k}\sim\gamma^*}[\Vert \boldsymbol g(\boldsymbol y_{t_k},t_k) - \boldsymbol f(\boldsymbol y_{t_k}, t_k,\boldsymbol \theta)\Vert] \\
    &\qquad \qquad \qquad+ L \mathbb{E}_{\boldsymbol x_{t_k},\boldsymbol y_{t_k}\sim\gamma^*}[\Vert \boldsymbol y_{t_k} - \boldsymbol x_{t_k}\Vert]\\
    {=}&\mathbb{E}_{\boldsymbol y_{t_k}\sim p_{t_k}}[\Vert \boldsymbol g(\boldsymbol y_{t_k},t_k) - \boldsymbol f(\boldsymbol y_{t_k}, t_k,\boldsymbol \theta)\Vert] + L\mathcal{W}[q_{t_k},p_{t_k}].
\end{align*}
Here, (\rom{1}) holds because $\boldsymbol f$ satisfies the Lipschitz condition. 

\vspace{3mm}
\noindent Details for \cref{E2}:
\begin{align*}
    &\mathbb{E}_{\boldsymbol y_{t_k}\sim p_{t_k}}[\Vert \boldsymbol g(\boldsymbol y_{t_k},t_k) - \boldsymbol f(\boldsymbol y_{t_k}, t_k,\boldsymbol \theta)\Vert]\\
   \overset{(\rom 1)}{=}&\mathbb{E}_{\boldsymbol y_{t_{k}}\sim p_{t_{k}}}[\Vert \boldsymbol g(\boldsymbol y_{t_{k-1}},t_{k-1}) - \boldsymbol f(\boldsymbol y_{t_{k-1}}, t_{k-1},\boldsymbol \theta)\\
    \quad&+ \boldsymbol f(\boldsymbol y_{t_{k-1}}, t_{k-1},\boldsymbol \theta) - \boldsymbol f(\boldsymbol y_{t_{k-1}}^\phi, t_{k-1},\boldsymbol \theta) \\
    &\quad+ \boldsymbol f(\boldsymbol y_{t_{k-1}}^\phi, t_{k-1},\boldsymbol \theta) - \boldsymbol f(\boldsymbol y_{t_k}, t_k,\boldsymbol \theta)\Vert]\\
    \leq&\mathbb{E}_{\boldsymbol y_{t_{k}}\sim p_{t_{k}}}[\Vert \boldsymbol g(\boldsymbol y_{t_{k-1}},t_{k-1}) - \boldsymbol f(\boldsymbol y_{t_{k-1}}, t_{k-1},\boldsymbol \theta)\Vert]\\
    &\quad+ \mathbb{E}_{\boldsymbol y_{t_{k}}\sim p_{t_{k}}}[\Vert \boldsymbol f(\boldsymbol y_{t_{k-1}}, t_{k-1},\boldsymbol \theta) - \boldsymbol f(\boldsymbol y_{t_{k-1}}^\phi, t_{k-1},\boldsymbol \theta)\Vert] \\
    &\quad+\mathbb{E}_{\boldsymbol y_{t_{k}}\sim p_{t_{k}}}[\Vert \boldsymbol f(\boldsymbol y_{t_{k-1}}^\phi, t_{k-1},\boldsymbol \theta) - \boldsymbol f(\boldsymbol y_{t_k}, t_k,\boldsymbol \theta)\Vert]\\
    \overset{(\rom 2)}{\leq}&\mathbb{E}_{\boldsymbol y_{t_{k}}\sim p_{t_{k}}}[\Vert \boldsymbol g(\boldsymbol y_{t_{k-1}},t_{k-1}) - \boldsymbol f(\boldsymbol y_{t_{k-1}}, t_{k-1},\boldsymbol \theta)\Vert]\\
    &\quad+ L\Vert \boldsymbol y_{t_{k-1}} - \boldsymbol y_{t_{k-1}}^\phi\Vert \\
    &\quad+\mathbb{E}_{\boldsymbol y_{t_{k}}\sim p_{t_{k}}}[\Vert\boldsymbol  f(\boldsymbol y_{t_{k-1}}^\phi, t_{k-1},\boldsymbol \theta) - \boldsymbol f(\boldsymbol y_{t_k}, t_k,\boldsymbol \theta)\Vert]\\
    \overset{(\rom 3)}{=}&\mathbb{E}_{\boldsymbol y_{t_{k-1}}\sim p_{t_{k-1}}}[\Vert \boldsymbol g(\boldsymbol y_{t_{k-1}},t_{k-1}) - \boldsymbol f(\boldsymbol y_{t_{k-1}}, t_{k-1},\boldsymbol \theta)\Vert]\\
    &\quad+  L(t_{t_k}-t_{k-1})O(t_{t_k}-t_{k-1})\\
    &\quad+\mathbb{E}_{\boldsymbol y_{t_{k}}\sim p_{t_{k}}}[\Vert \boldsymbol f(\boldsymbol y_{t_{k-1}}^\phi, t_{k-1},\boldsymbol \theta) - \boldsymbol f(\boldsymbol y_{t_k}, t_k,\boldsymbol \theta)\Vert]\\
\end{align*}

Here, (\rom{1}) holds because $\boldsymbol g$ is a consistency function, with $\boldsymbol g(\boldsymbol y_{t_k},t_k)=\boldsymbol g(\boldsymbol y_{t_{k-1}},t_{k-1})$. (\rom{2}) holds because $\boldsymbol f$ satisfies the Lipschitz condition. (\rom{3}) holds because $\Phi$ is an Euler solver, hence $\Vert \boldsymbol y_{t_{k-1}} - \boldsymbol y_{t_{k-1}}^\phi\Vert$ does not exceed the truncation error $O((t_{n}-t_{n-1})^2)$.

\section{Conditional Discriminator}

\begin{theorem}
    Given a generator $G(\boldsymbol z,\boldsymbol x_t,t)$ and a discriminator $D(\boldsymbol x_0,\boldsymbol x_t, t)$. The distribution of optimal solution of $G(\cdot, \boldsymbol x_t,t)$ for the problem \cref{minmax} is $p_g(\cdot|\boldsymbol x_t)=p(\cdot|\boldsymbol x_t)$, where $p_g(\cdot|\boldsymbol x_t)$ is the sample distribution of $G(\boldsymbol z,\boldsymbol x_t,t), z\sim p_{\boldsymbol z}(\boldsymbol z|\boldsymbol x_t)$. $p_{\boldsymbol z}(\cdot|\boldsymbol x_t)$ is a normal distribution. $\boldsymbol x_t\sim p_t$, and $\boldsymbol x_0 \sim p_0$. $p_t$ is the marginal distribution of a diffusion process.
    
\begin{equation}
\begin{split}
\min _{G}& \max _{D} V(G, D)=\mathbb{E}_{\boldsymbol x_0,\boldsymbol x_t \sim p(\boldsymbol x_0,\boldsymbol x_t)}[\log D(\boldsymbol x_0, \boldsymbol x_t)]\\
                    &+\mathbb{E}_{\boldsymbol z \sim p_{\boldsymbol z}(\boldsymbol z|\boldsymbol x_t),\boldsymbol x_t\sim p_t}[\log (1-D(G(\boldsymbol z, \boldsymbol x_t,t),\boldsymbol  x_t))]
\end{split}
\label{minmax}
\end{equation}

\end{theorem}

\begin{proof}
By expressing \cref{minmax} in integral form, we have the following equation:
\begin{align*}
 &\iint_{\boldsymbol x_0,\boldsymbol x_t}p(\boldsymbol x_0,\boldsymbol x_t)\log(D(\boldsymbol x_0,\boldsymbol x_t))d\boldsymbol x_0d\boldsymbol x_t\\
 &+\iint_{\boldsymbol z,\boldsymbol x_t}p_{\boldsymbol z}(\boldsymbol z,\boldsymbol x_t)\log(1-D(G(\boldsymbol z,\boldsymbol x_t),\boldsymbol x_t))d\boldsymbol zd\boldsymbol x_t\\
 =&\int_{\boldsymbol x_t}p_t(\boldsymbol x_t)\left(\int_{\boldsymbol x_0} p(\boldsymbol x_0|\boldsymbol x_t)\log(D(\boldsymbol x_0,\boldsymbol x_t))d\boldsymbol x_0\right. \\
 &+ \left.\int_{\boldsymbol z}p_{\boldsymbol z}(\boldsymbol z|\boldsymbol x_t)\log(1-D(G(\boldsymbol z,\boldsymbol x_t),\boldsymbol x_t))d\boldsymbol z\right)d\boldsymbol x_t\\
 =&\mathbb E_{\boldsymbol x_t\sim p_t}\left[\int_{\boldsymbol x_0} p(\boldsymbol x_0|\boldsymbol x_t)\log(D(\boldsymbol x_0,\boldsymbol x_t))\right.\\
 &+ \left. p_g(\boldsymbol x_0|\boldsymbol x_t)\log(1-D(\boldsymbol x_0,\boldsymbol x_t))d\boldsymbol x_0\right]
\end{align*}

\noindent The optimal $D$ is:
$$
D_G^*=\frac{p(\boldsymbol x_0|\boldsymbol x_t)}{p(\boldsymbol x_0|\boldsymbol x_t)+p_g(\boldsymbol x_0|\boldsymbol x_t)}
$$
\noindent Substituting $D^*$ into $V$, we obtain the following equation:
\begin{align*}
&\max_D V(G,D)\\
=&\mathbb E_{\boldsymbol x_t\sim p_t}\left[ \mathbb E_{\boldsymbol x_0\sim p(\boldsymbol x_0|\boldsymbol x_t)}\left[\log\frac{p(\boldsymbol x_0|\boldsymbol x_t)}{p(\boldsymbol x_0|\boldsymbol x_t)+p_g(\boldsymbol x_0|\boldsymbol x_t)}\right]\right.\\
&+ \left. \mathbb E_{\boldsymbol x_0\sim p_g(\boldsymbol x_0|\boldsymbol x_t)}\log \left[ \frac{p_g(\boldsymbol x_0|\boldsymbol x_t)}{p(\boldsymbol x_0|\boldsymbol x_t)+p_g(\boldsymbol x_0|\boldsymbol x_t)} \right] \right]\\
=&\mathbb E_{\boldsymbol x_t\sim p_t}\left[ -\log 4 + 2 \textit{JSD}(p_t(\cdot|\boldsymbol x_t)||p_g(\cdot| \boldsymbol x_t)) \right]
\end{align*}
In the aforementioned equation, $\textit{JSD}$ represents the Jensen-Shannon divergence. The equation holds true only when $p_g(\cdot|\boldsymbol  x_t)=p(\cdot|\boldsymbol  x_t)$. This concludes the proof.
\end{proof}

\label{condition_discriminator}

\section{ACT-Aug}

\label{sec:act}

In this section, we will provide the details of ACT-Aug. The differences from ACT are highlighted in red. The algorithm is listed in \cref{act_aug}.

\begin{algorithm}[h!]
\caption{Adversarial Consistency Training with Augmentation}
\label{act_aug}
\begin{algorithmic}[1]
\State \textbf{Input:} dataset $\mathcal{D}$, initial consistency model parameter $\theta_g$, discriminator $\theta_d$, step schedule $N(\cdot)$, EMA decay rate schedule $\mu(\cdot)$, optimizer $\text{opt}(\cdot,\cdot)$, discriminator with augmentation \textcolor{red}{$D_{aug}(\cdot,\cdot,\cdot,\theta_d)$}, adversarial rate schedule $\lambda(\cdot)$, gradient penalty weight $w_{gp}$, gradient penalty interval $I_{gp}$, \textcolor{red}{gradient penalty threshold $\tau$}, \textcolor{red}{augmentation probability update rate $p_r$}

\State $\boldsymbol \theta_g^-\leftarrow \boldsymbol \theta$, $k\leftarrow 0$, \textcolor{red}{$p_{aug}\leftarrow 0$ and $\mathcal{L}_{gp}^-=\tau$}
\Repeat

\State Sample $\boldsymbol x\sim \mathcal{D}$, and $n\sim\mathcal{U}[\![1,N(k)]\!]$
\State Sample $\boldsymbol z\sim \mathcal{N}(0,\boldsymbol I)$ \Comment{Train Consistency Model}
\State $\mathcal{L}_{CT}\leftarrow$
\Statex \hspace{\algorithmicindent} \quad $d(\boldsymbol f(\boldsymbol x+t_{n+1}\boldsymbol z,t_{n+1}, \boldsymbol \theta_g), \boldsymbol f(\boldsymbol x+t_n\boldsymbol z,t_n,\boldsymbol \theta_g^-))$
\State $\mathcal L_G\leftarrow\log(1-$
\Statex \hspace{\algorithmicindent} \quad $\textcolor{red}{D_{aug}}(\boldsymbol f(\boldsymbol x+t_{n+1}\boldsymbol z,t_{n+1}, \textcolor{red}{p_{aug}},\boldsymbol \theta_g), t_{n+1}, \boldsymbol \theta_d))$
\State $\mathcal{L}_f\leftarrow(1-\lambda_{N(k)}(n+1))\mathcal{L}_{CT} + \lambda_{N(k)}(n+1)\mathcal L_G$
\State $\boldsymbol \theta_g\leftarrow\text{opt}(\boldsymbol \theta_g,\nabla_{\boldsymbol \theta_g}(\mathcal{L}_f))$
\State $\boldsymbol \theta_g^-\leftarrow\text{stopgrad}(\mu(k)\boldsymbol \theta_g^-+(1-\mu(k))\boldsymbol \theta_g)$

\Statex

\State Sample $\boldsymbol x_g\sim \mathcal{D}$, $\boldsymbol x_r\sim \mathcal{D}$, and $n\sim\mathcal{U}[\![1,N(k)]\!]$
\State Sample $\boldsymbol z\sim \mathcal{N}(0,\boldsymbol I)$\Comment{Train Discriminator}
\State $\mathcal L_D\leftarrow-\log(\textcolor{red}{D_{aug}}(\boldsymbol x_r,t_{n+1}, \textcolor{red}{p_{aug}},\boldsymbol \theta_d))$
\Statex \hspace{\algorithmicindent} \quad $-\log(1-\textcolor{red}{D_{aug}}(\boldsymbol f(\boldsymbol x_g+t_{n+1}\boldsymbol z,t_{n+1}, \textcolor{red}{p_{aug}},\boldsymbol \theta_d))$
\State $\mathcal L_{gp}\leftarrow w_{gp}[k \mod I_{gp}=0]*$
\Statex \hspace{\algorithmicindent} \quad $\Vert\nabla_{\boldsymbol x_r} \textcolor{red}{D_{aug}}(\boldsymbol x_r,t_{n+1}, \textcolor{red}{p_{aug}},\boldsymbol \theta_d)\Vert^2$
\State $\mathcal{L}_d\leftarrow \lambda_{N(k)}(n+1)\mathcal{L}_D + \lambda_{N(k)}(n+1)\mathcal{L}_{gp}$
\State $\boldsymbol \theta_d\leftarrow\text{opt}(\boldsymbol \theta_d,\nabla_{\boldsymbol \theta_d}(\mathcal{L}_d))$
\If {\textcolor{red}{$k\mod I_{gp}=0$}}
\State \textcolor{red}{$p_{aug}\leftarrow$}
\Statex \hspace{\algorithmicindent} \hspace{\algorithmicindent} \quad \textcolor{red}{$ \text{Clip}_{[0, 1]}(p_{aug}+2([\mathcal L_{gp}^- >= \tau] - 0.5)p_r)$}
\State \textcolor{red}{$\mathcal L_{gp}^-=\mu_p\mathcal L_{gp}^-+(1-\mu_p)\mathcal L_{gp}$}
\EndIf
\State $k\leftarrow k+1$
\Until{convergence}
\end{algorithmic}
\end{algorithm}

\section{More Experiment Results}

\label{inpainting_sec}
\textbf{Zero-shot Image Inpainting} ~ An important capability of consistency models is zero-shot image inpainting. This depends on the properties of the diffusion process and $\mathcal L_{CT}$. Given that we introduce a discriminator during the training process, does this impact the properties of consistency models? We demonstrate the results of inpainting in \cref{inpainting}. We employ the algorithm consistent with \cite{Song2023Consistency}. It can be seen that ACT still retains the capabilities of consistency models.

We further display the sampling results from the conditional trajectory $\{\boldsymbol{x}_0 + t_{k}\boldsymbol{z}\}, \boldsymbol x_0\sim p_0,\boldsymbol z\sim \mathcal N ({0,\boldsymbol I})$ on ImageNet 64$\times$64. $k$ ranges from $0$ to $N$, with $10$ equidistant points. It can be observed that the sampling results of $t_k$ and $t_{k-1}$ exhibit significant similarity, which further substantiates that ACT does not disrupt the properties of $\mathcal L_{CT}$ and consistency models.

\begin{figure}[h!]
    \centering
    \includegraphics[width=0.95\columnwidth]{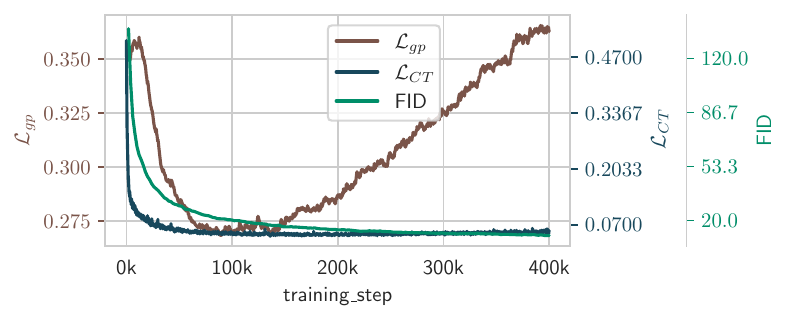}
    \caption{$\mathcal L_{gp}$, $\mathcal{L}_{CT}$, and FID of ACT on ImageNet 64x64 ($w_{mid=0.2}, w=0.6$, a suitable parameter set. Under these parameters, all three metrics demonstrate stability).}
    \label{proper_lambda_fig}
\end{figure}

\begin{figure}[h!]
    \centering
    \includegraphics[width=0.95\columnwidth]{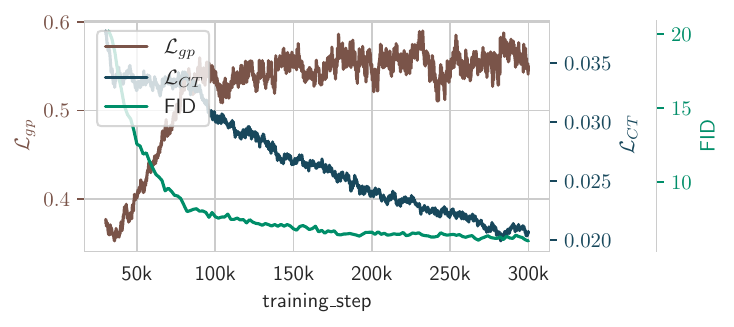}
    \caption{$\mathcal L_{gp}$, $\mathcal{L}_{CT}$, and FID of ACT-Aug on CIFAR10 ($\lambda_N\equiv 0.3$, a suitable parameter set. Under these parameters, all three metrics demonstrate stability).}
    \label{proper_lambda_fig_2}
\end{figure}

\begin{figure}
        \centering 
    \begin{subfigure}{0.95\columnwidth}
        \centering 
        \includegraphics[width=.95\linewidth]{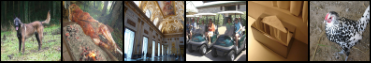}
    \end{subfigure}
    
    \begin{subfigure}{0.95\columnwidth}
        \centering
        \includegraphics[width=.95\linewidth]{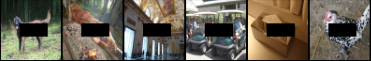}
    \end{subfigure}

    \begin{subfigure}{0.95\columnwidth}
        \centering
        \includegraphics[width=.95\linewidth]{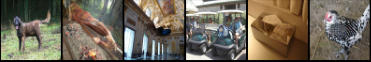}
    \end{subfigure}
    
\caption{The results of zero-shot inpainting. \textbf{First Row:} original images; \textbf{Second Row:} masked images; \textbf{Bottom Row:} inpainted images.}
    \vspace{-3mm}
\label{inpainting}
\end{figure}

\vspace{3mm}
\noindent\textbf{Generation Visualization on Conditional Trajectory} ~ In this section, we demonstrate samples generated from the conditional trajectory $\{\boldsymbol{x}_0 + t_{k}\boldsymbol{z}\}$ on ImageNet 64$\times$64, further illustrating that our method preserves the properties of consistency training. \cref{sample_noise} shows the conditional trajectory $\{\boldsymbol{x}_0 + t_{k}\boldsymbol{z}\}$, while \cref{sample_t_img} displays the samples generated from the conditional trajectory $\{\boldsymbol{x}_0 + t_{k}\boldsymbol{z}\}$. It can be observed that there is a high degree of similarity between adjacent $t$ values, further validating that our method retains the properties of $\mathcal L_{CT}$.

\label{sample_t}

\begin{figure}
    \centering
    \includegraphics[width=0.95\columnwidth]{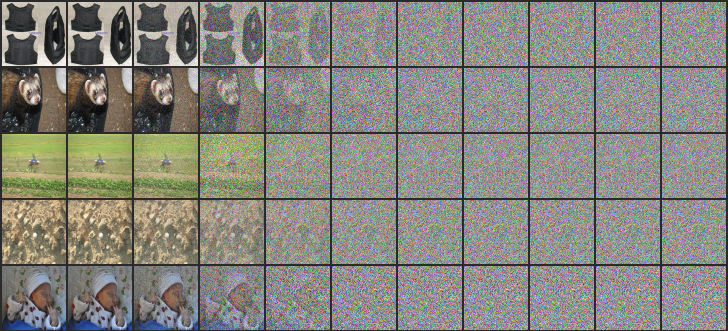}
    \caption{The conditional trajectory $\{\boldsymbol{x}_0 + t_{k}\boldsymbol{z}\}$ (ImageNet 64$\times$64).}
    \label{sample_noise}
\end{figure}

\begin{figure}
    \centering
    \includegraphics[width=0.95\columnwidth]{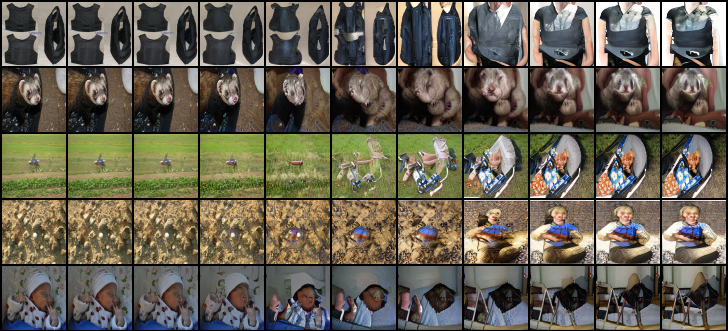}
    \caption{Generated from the conditional trajectory $\{\boldsymbol{x}_0 + t_{k}\boldsymbol{z}\}$ (ImageNet 64$\times$64).}
    \label{sample_t_img}
\end{figure}

\label{proper_lambda}
\vspace{3mm}
\noindent\textbf{Examples of proper \texorpdfstring{$\lambda_N$}{lambda}} ~ In this section, we present the stability of $\mathcal{L}_{CT}$, $\mathcal{L}_{gp}$, and the FID score of the appropriate selection of $\lambda_N$. As depicted in \cref{proper_lambda_fig}, it is observed that all three metrics exhibit stability during training. Specifically for $\mathcal{L}_{gp}$, there is an initial decreasing trend followed by an increase; however, the variation remains within a range of $0.1$ until the end of training. 

\cref{proper_lambda_fig_2} illustrates the stability of $\mathcal{L}_{gp}$, $\mathcal{L}_{CT}$, and the FID score for ACT-Aug under the appropriate selection of $\lambda_N$. It is observed that all three metrics exhibit stability. Furthermore, when compared with ACT on CIFAR10 as shown in \cref{cifar10_noada}, $\mathcal{L}_{gp}$ is stabilized around the set $\tau=0.55$, and both $\mathcal{L}_{CT}$ and the FID score continue to show a decreasing trend. This validates the effectiveness of the augmentation.

\vspace{3mm}
\noindent\textbf{More samples.} ~ \cref{mode_collapse} shows failed generations on CIFAR10 dataset. \cref{cat_1,cat_2} shows more samples on LSUN Cat 256$\times$256 dataset.

\begin{figure}
        \centering 
    \begin{subfigure}[t]{0.235\textwidth}
        \centering 
        \includegraphics[width=.95\linewidth]{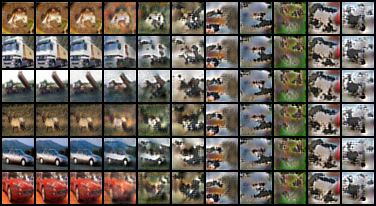}
    \captionsetup{skip=0pt, width=.95\linewidth}
      \subcaption[]{Generated from the conditional trajectory $\{\boldsymbol{x}_0 + t_{k}\boldsymbol{z}\}$.}
        \label{ablation_a}
    \end{subfigure}
    \begin{subfigure}[t]{0.235\textwidth}
        \centering
        \includegraphics[width=.95\linewidth]{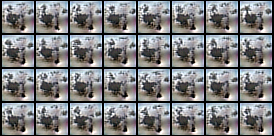}
        \captionsetup{skip=0pt, width=.95\linewidth}
        \subcaption[]{Sampling from $T\boldsymbol{z}$.}
        \label{ablation_b}
    \end{subfigure}
    \vspace{-2mm}
    
\caption{Failed generations. Mode collapse when $\lambda_N\approx 1$. Experiments are conducted on the CIFAR10 dataset.}
    \vspace{-3mm}
        \label{mode_collapse}
\end{figure}

\begin{figure*}
    \centering
    \includegraphics[width=0.95\textwidth]{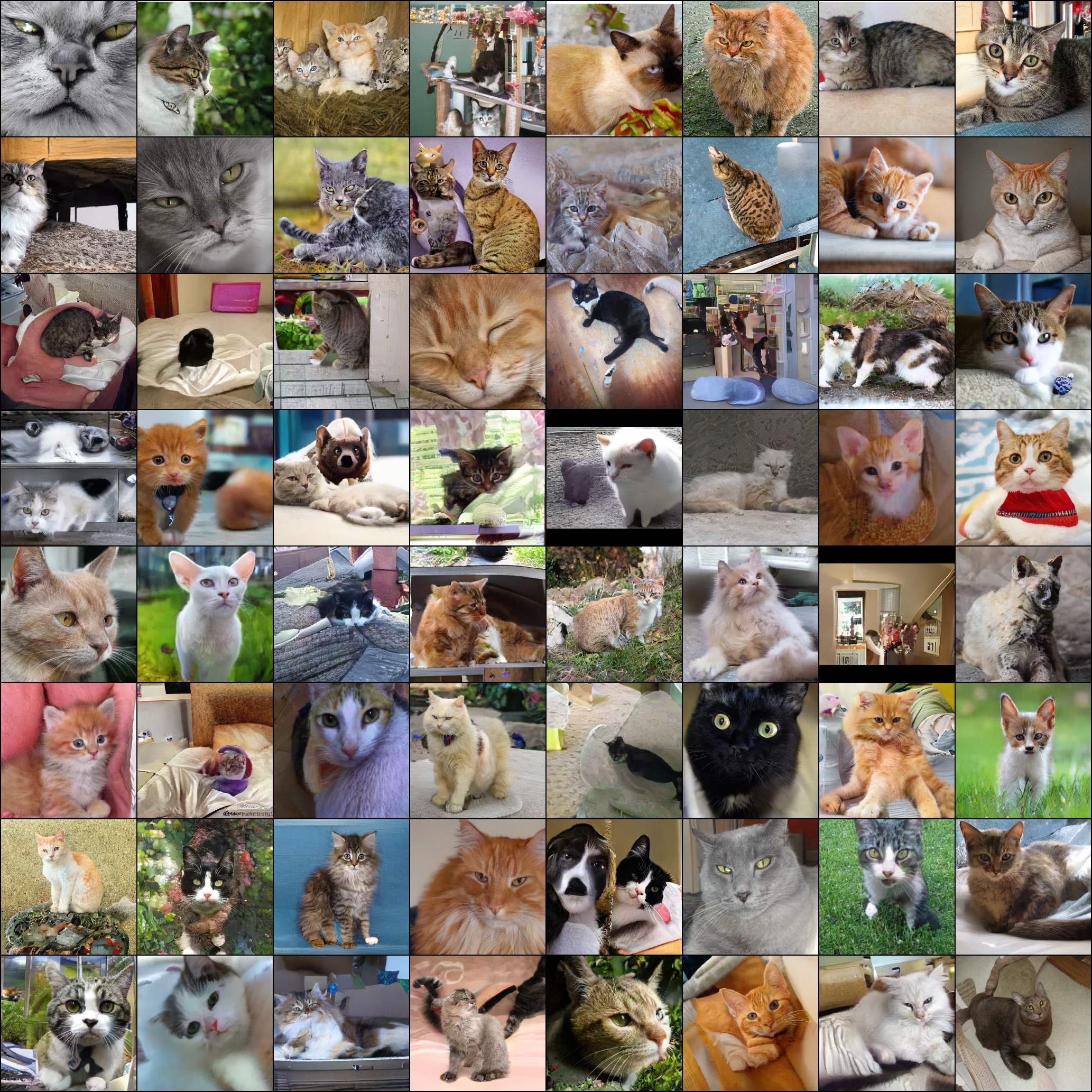}
    \label{cat_1}
\end{figure*}

\begin{figure*}
    \centering
    \includegraphics[width=0.95\textwidth]{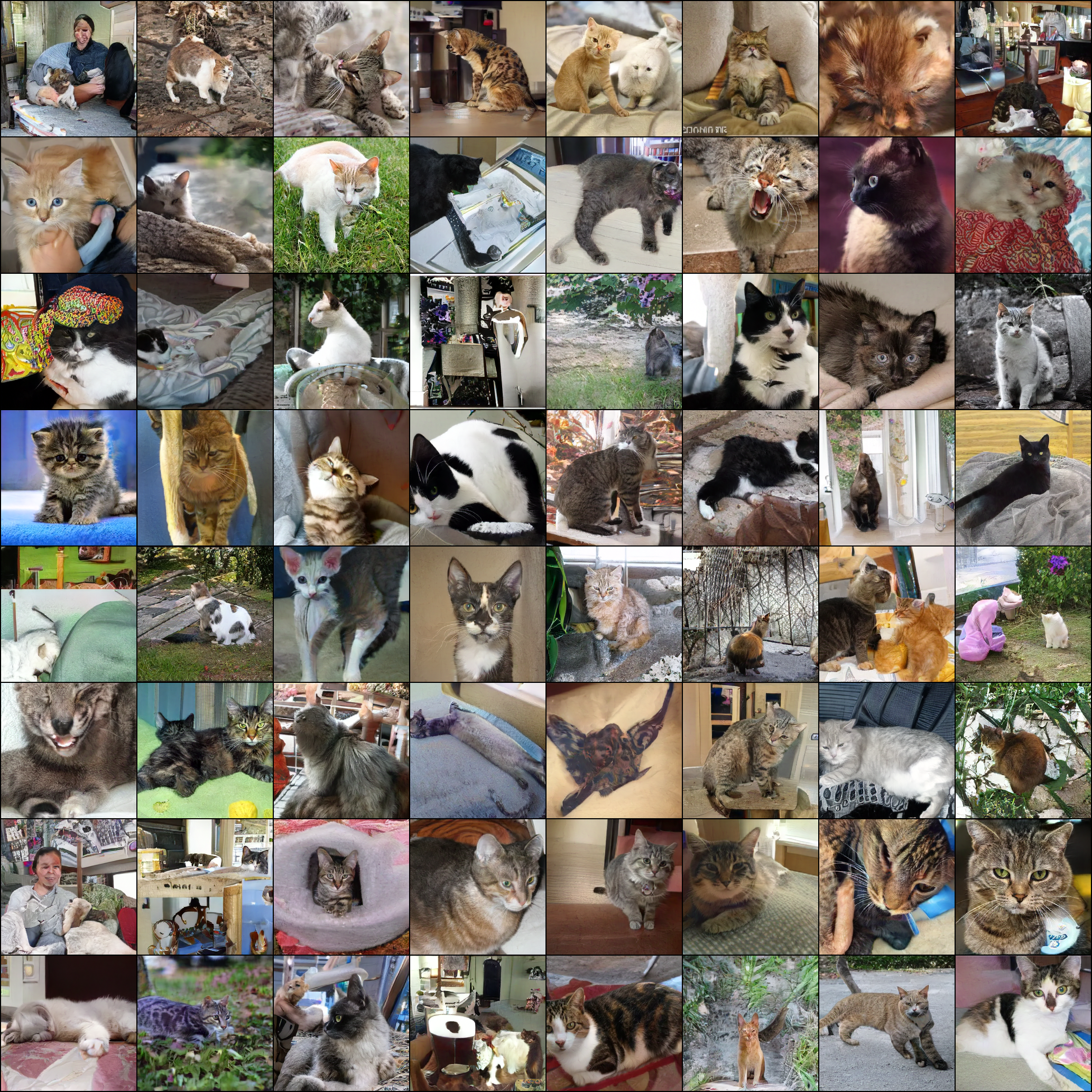}
    \caption{Generated samples (ACT Trained on LSUN Cat 256$\times$256).}
    \label{cat_2}
\end{figure*}

\end{document}

%% file: _constants.tex
\def\paperID{0000}
\def\confName{CVPR}
\def\confYear{2024}

\newif\ifreview 
\newif\ifarxiv \newcommand{\arxiv}{\arxivtrue}
\newif\ifcamera 
\newif\ifrebuttal 

%% file: cvpr_header.tex
\ifreview \usepackage[review]{cvpr} \fi
\ifarxiv \usepackage[pagenumbers]{cvpr} \fi
\ifrebuttal \usepackage[rebuttal]{cvpr} \fi
\ifcamera \usepackage{cvpr} \fi

\input{_macros}  

\usepackage{xr-hyper}

\makeatletter
\newcommand*{\addFileDependency}[1]{
  \typeout{(#1)}
  \@addtofilelist{#1}
  \IfFileExists{#1}{}{\typeout{No file #1.}}
}

\makeatother

\definecolor{cvprblue}{rgb}{0.21,0.49,0.74}
\usepackage[pagebackref,breaklinks,colorlinks,citecolor=cvprblue]{hyperref}
\usepackage[capitalize]{cleveref}
\crefname{section}{Sec.}{Secs.}
\crefname{table}{Table}{Tables}
\crefname{figure}{Fig.}{Figs.}

%% file: _macros.tex
\usepackage{graphicx}	
\usepackage{amsmath}	
\usepackage{amssymb}	
\usepackage{booktabs}
\usepackage{times}
\usepackage{microtype}
\usepackage{epsfig}
\usepackage[table,xcdraw,dvipsnames]{xcolor}
\usepackage{caption}
\usepackage{float}
\usepackage{placeins}
\usepackage{color, colortbl}
\usepackage{stfloats}
\usepackage{enumitem}
\usepackage{tabularx}
\usepackage{xstring}
\usepackage{multirow}
\usepackage{xspace}
\usepackage{url}
\usepackage{subcaption}
\usepackage{xcolor}
\usepackage[hang,flushmargin]{footmisc}

\ifcamera \usepackage[accsupp]{axessibility} \fi

% ...

\ifarxiv  \fi

\newcommand{\R}[1]{{%
    \textbf{%
        \ifstrequal{#1}{1}{\textcolor{red}{R#1}}{%
        \ifstrequal{#1}{2}{\textcolor{blue}{R#1}}{%
        \ifstrequal{#1}{3}{\textcolor{magenta}{R#1}}{%
        \ifstrequal{#1}{4}{\textcolor{teal}{R#1}}{%
                           \textcolor{cyan}{R#1}%
        }}}}%
    }%
}}